\documentclass{article}

\usepackage[preprint]{neurips_2024}

\usepackage[utf8]{inputenc} %
\usepackage[T1]{fontenc}    %
\usepackage[colorlinks=true,allcolors=blue!60!black]{hyperref}       %
\usepackage{url}            %
\usepackage{booktabs}       %
\usepackage{amsfonts}       %
\usepackage{nicefrac}       %
\usepackage{microtype}      %
\usepackage{xcolor}         %
\usepackage[T1]{fontenc}
\usepackage[utf8]{inputenc}
\usepackage{microtype}

\usepackage{algorithm}
\usepackage{algorithmic}

\usepackage{natbib}
\usepackage{graphicx} %
\usepackage{mathrsfs}
\usepackage{paralist}
\usepackage[super]{nth}
\usepackage[toc,page]{appendix}
\usepackage{listings}
\usepackage{mathtools}
\usepackage{color}
\usepackage{tabularx}
\usepackage{amsmath}
\usepackage{amsthm}
\usepackage{csquotes}
\usepackage{multicol}
\usepackage{amssymb}
\usepackage{fancyvrb}
\usepackage{makecell}
\usepackage{dsfont}
\usepackage{graphicx} %
\usepackage{mathrsfs}
\usepackage{amsthm}
\usepackage[super]{nth}
\usepackage[toc,page]{appendix}
\usepackage{listings}
\usepackage{mathtools}
\usepackage{color}
\usepackage{tabularx}
\usepackage{amsmath}
\usepackage{csquotes}
\usepackage{multicol}
\usepackage{amssymb}
\usepackage{fancyvrb}
\usepackage{makecell}
\usepackage{dsfont}
\usepackage{url}

\usepackage{mdframed}
\usepackage[toc,page]{appendix}
\usepackage{mathdots}
\usepackage{xspace}
\usepackage{empheq}
\usepackage{arydshln}
\usepackage{nicefrac}
\usepackage{dsfont}
\usepackage[capitalize,noabbrev]{cleveref}
\crefname{figure}{Figure}{Figures}

\usepackage{array}
\newcolumntype{L}[1]{>{\raggedright\let\newline\\\arraybackslash\hspace{0pt}}m{#1}}
\newcolumntype{C}[1]{>{\centering\let\newline\\\arraybackslash\hspace{0pt}}m{#1}}
\newcolumntype{R}[1]{>{\raggedleft\let\newline\\\arraybackslash\hspace{0pt}}m{#1}}

\usepackage{tikz}
\usetikzlibrary{automata, positioning, calc, shapes, arrows, fit, decorations.text, quotes}

\DeclareMathOperator*{\argmax}{arg\,max}
\DeclareMathOperator*{\argmin}{arg\,min}

\newcommand{\rjr}{mJR\xspace}
\newcommand{\rpjr}{mPJR\xspace}

\usepackage{thmtools}
\theoremstyle{plain}
\newtheorem{theorem}{Theorem}

\newtheorem{lemma}[theorem]{Lemma}

\newtheorem{proposition}[theorem]{Proposition}
\newtheorem{observation}[theorem]{Observation}

\theoremstyle{definition}
\newtheorem{definition}{Definition}
\theoremstyle{remark}

\declaretheorem[style=remark,name=Example,qed=$\diamond$]{example}

\title{Proportional Fairness in Clustering: A~Social~Choice~Perspective}

\author{%
  Leon~Kellerhals\\
  Technische Universität Berlin\\
  \texttt{leon.kellerhals@tu-berlin.de} \\
  \And
  Jannik~Peters\\
  Technische Universität Berlin\\
  \texttt{jannik.peters@tu-berlin.de} \\
}

\begin{document}

\maketitle

\begin{abstract}
  We study the proportional clustering problem of \citeauthor{CFL+19a} (ICML'19) and relate it to the area of multiwinner voting in computational social choice. We show that any clustering satisfying a weak proportionality notion of \citeauthor{BrPe23a} (EC'23) simultaneously obtains the best known approximations to the proportional fairness notion of \citeauthor{CFL+19a}, but also to individual fairness (\citeauthor{JKL20a}, FORC'20) and the ``core'' (\citeauthor{LLS+21a}, ICML'21). In fact, we show that any approximation to proportional fairness is also an approximation to individual fairness and vice versa. Finally, we also study stronger notions of proportional representation, in which deviations do not only happen to single, but multiple candidate centers, and show that stronger proportionality notions of \citeauthor{BrPe23a} imply approximations to these stronger guarantees. 
\end{abstract}

\section{Fair clustering}

Fair decision-making is a crucial research area in artificial intelligence and machine learning. To ensure fairness, a plethora of different fairness notions, algorithms and settings have been introduced, studied, and implemented. One area in which fairness has been applied extensively is \emph{(centroid) clustering}:
We are given a set of $n$ data points which we want to partition into $k$ clusters by choosing $k$ ``centers'' and assigning each point to a center by which it is \emph{represented well}.
Fairness now comes into play when, e.g., the data points correspond to human individuals.

Fairness notions in clustering usually depend on one decision: whether one takes demographic information (such as gender, income, etc.) into account or whether one is agnostic to it.
A large part of work on fair clustering has focused on incorporating such demographic information, starting with the seminal work of \citet{CKLV17a} who aimed to proportionally balance the number of people of a certain type in each cluster center. 
However, not all work on fair clustering relies on demographic information.
Independently, and in different contexts, \citet*{JKL20a} and \citet*{CFL+19a} instead tried to derive fairness notions from the instance itself.
For \citeauthor{JKL20a} this lead to their notion of \emph{individual fairness}:
Given a population of size $n$, with $k$ cluster centers to be opened, every agent should be entitled to a cluster center not further away than their $\frac{n}{k}$-th neighbor.
While this is not always achievable, \citeauthor{JKL20a} gave a simple algorithm achieving a $2$-approximation to this notion. 
\citeauthor{CFL+19a} were motivated not by being fair towards individual members of the population (or agents), but towards groups of agents, defining their notion of \emph{proportional fairness}:
no group of size at least $\frac{n}{k}$ should be able to suggest a cluster center they all would be better off with.
This notion is also not always achievable, and \citeauthor{CFL+19a} gave a simple $(1 + \sqrt{2})$-approximation for it.

So far, the individual and proportional fairness notions (and some other related fairness notions) have existed in parallel, with similarities between the two being acknowledged but not formalized.\footnote{For instance, in a recent tutorial \citep{ABFair} on fair clustering, the two notions were treated as separate unconnected paradigms.}
In their survey, \citet{DEMZ24a} highlight this as a general issue in fair clustering: ``each notion that was introduced [...] does not refer to or consider the interaction with the previously introduced fairness notions in clustering''.
Moreover, they call for ``other fairness notions in clustering that are also compatible with one another'' and ``general notions which possibly encompass existing ones''.

We follow this call and prove proportional and individual fairness, as well as a fairness notion by \citet{LLS+21a} which we will call the \emph{transferable core}, to be tightly related to another.
In an effort encompass the three notions,
we make use of proportionality axioms from \emph{multiwinner voting}, an area in Computational Social Choice \citep{LaSk22a}.
Here, given the votes of $n$ agents, the goal is to elect a size-$k$ committee which fulfills some proportionality guarantee.
We lift one of the simplest proportionality guarantees (JR) to work with metric distances and prove that any clustering fulfilling our guarantee also fulfills the best approximations for the three notions, all \emph{simultaneously}.
Moreover, such a clustering can be computed in polynomial time.
Taking the multiwinner voting approach further, we also look at the lifted version of a stronger proportionality guarantee (PJR).
This changes how points (agents) interact with cluster centers as they become represented not by one, but possibly multiple centers.
While this is not standard for ``vanilla'' clustering, it is very fitting for more democratic settings, where the chosen ``centers'' end up possessing voting power to represent the agents.
The resulting proportionality guarantee indeed highly relates to work by \citet{EbMi23a} who, motivated by \emph{sortition} (the randomized selection of citizens' panels \citep{FGG+21a}), introduced a generalization of the proportional fairness notion.
Indeed, the multiwinner voting perspective allows us to prove better approximation guarantees for their fairness notion.

\paragraph{Our contributions.}
As our first main result, we provide a simple bridge between proportional fairness and individual fairness (see \Cref{sec:bridges}). Any approximation of the former is also an approximation of the latter. In particular, for any $\alpha, \beta \ge 1$ we show that (i) any $\alpha$-approximation to proportional fairness is also an $(1 + \alpha)$-approximation to individual fairness and (ii) any $\beta$-approximation to individual fairness is also a $2\beta$-approximation to proportional fairness.
These approximations are tight.
We also prove a similar connection between proportional fairness and the transferable core. 
Our connections imply for instance that bi-criteria approximations that optimize $k$-means and, say, individual fairness \citep{VaYa22b,BCEL24a} also maintain approximations guarantees to the other fairness notions.
Further, if one wants to show incompatibility with a different clustering notion with approximate proportional or individual fairness, it is sufficient to show this for one of the two notions, instead of creating instances for both (as done by \citet{DEMZ23a}). 

Secondly, in \Cref{sec:jrpjr}, we draw a connection to the area of multiwinner voting and reinterpret proportionality notions introduced by \citet{BrPe23a} to work with distance metrics; we call the resulting guarantees \rjr and \rpjr. Both of these are efficiently computable when the space of possible centers is finite.
Remarkably, with simple proofs, we are able to show that any clustering satisfying \rjr achieves the \emph{best known} approximations to individual and proportional fairness notions and the transferable core.
For the transferable core we even improve upon the bound derived by \citet{LLS+21a}.
Finally, motivated by settings such as sortition and multiwinner voting in which agents do not only care about their closest cluster center but are represented by multiple centers, we show that a strong core stability guarantee (introduced by \citet{EbMi23a}) can be achieved by any clustering satisfying \rpjr.
We also deal with the case in which the center candidate space is unbounded (e.g., in Euclidean clustering settings), in which the above-mentioned algorithms can become intractable.
Here, we show that satisfying the proportionality guarantees only for the set of agents is sufficient to obtain constant-factor approximations to proportional fairness and the core stability guarantee by \citet{EbMi23a}.

Lastly, in \cref{sec:sortition}, we focus on sortition: Here, the set of agents and cluster candidates is equal and each agent must be chosen with equal probability.
Employing techniques from the above results, we are able to give a simpler proof achieving a better approximation guarantee for the core notion by \citet{EbMi23a}. 

\cref{fig:results} (right) gives an overview over our results and the exact approximation guarantees.

\begin{figure}[t]
    \centering
    \hfill
    \begin{minipage}{0.778\linewidth}
    \begin{tikzpicture}[
        arc/.style={
        },
        every edge/.style = {draw,->,>=stealth,semithick},
        every edge quotes/.style = {auto, font=\footnotesize, sloped, inner sep=1pt},
    ]

        \node[draw] (rPJR)    at (3.55,2.3) {\rpjr};
        \node[draw] (rJR)    at (0.00,2.3) {\rjr};
        
        \node[draw,double,rounded corners,font=\scshape,above= of rPJR,yshift=-.5cm] (EA) {expanding approvals};
        \node[draw,double,rounded corners,font=\scshape,above= of rJR,yshift=-.5cm] (GC) {greedy capture};
        \node[draw,double,rounded corners,font=\scshape,right=of EA,xshift=-.86cm] (FGC) {fair greedy capture};
        
        \node[draw] (if) at ( -1.0,1.0) {$\beta$-IF};
        \node[draw] (pf)  at ( 2.0,1.0) {$\alpha$-PF};
        \node[draw] (tc) at ( 4.7,1.0) {$(\gamma,\alpha)$-TC};
        
        \node[draw] (qc) at ( 6.7,1.0) {$\alpha$-$q$-core};
        
        \draw (GC) edge (rJR);
        \draw (EA) edge (rPJR);
        \draw (rPJR) edge (rJR);
        
        \draw (pf) edge["{$\frac{\gamma(\alpha+1)}{\gamma-1}$}"'] (tc);
        \draw (pf) edge["{$1+\alpha$}", bend left=15] (if);
        \draw (if) edge["{$2\beta$}", pos=0.5, bend left=5] (pf);
        
        \draw (rJR)  edge["{$\frac{2\gamma}{\gamma-1}$}", pos=0.5] (tc);
        \draw (rJR)  edge["{$2.41$}"', pos=0.4] (pf);
        \draw (rJR)  edge["{$2$}", pos=0.4] (if);
        \draw (rPJR)  edge["{$5$}", pos=0.5] (qc);
	\draw (FGC)  edge["{$3.56$}", pos=0.5] (qc);
    \end{tikzpicture}
    \end{minipage}
    \hfill
    \begin{minipage}{0.094\linewidth}
    \begin{tikzpicture}[scale=0.92]
        \begin{scope}[rotate=270,scale=.7]
            \node [inner sep=1.5pt,align=center] (v1) at (.2,0.5) {$1\,\ 2$\\$3\,\,4$};
            
	    \node[inner sep=1.5pt] (v4) at (2,0.5) {$5$};
            
            \node[inner sep=1.5pt] (v5) at (3,0.5) {$6$};
            
            \node[inner sep=1.5pt] (v6) at (4,0.5) {$7$};
        
            \node[inner sep=1.5pt] (v7) at (2,-0.5) {$8$};
        
            \node[inner sep=1.5pt] (v8) at (3,-0.5) {$9$};
        
            \node[inner sep=1.5pt] (v9) at (4,-0.5) {$10$};
            
            \draw (v1) -- node[left,xshift=.2em,font=\small] {\footnotesize $10$} (v4);
            \draw (v4) -- (v5);
            \draw (v5) -- (v6);
            \draw (v4) -- (v7);
            \draw (v5) -- (v8);
            \draw (v6) -- (v9);
            \draw (v7) -- (v8);
            \draw (v8) -- (v9);
        \end{scope}
    \end{tikzpicture}
    \end{minipage}
    \hfill
    \caption{
        \emph{Left:}
        An overview over connections between and bounds on fairness notions, i.e., $\alpha$-proportional fairness ($\alpha$-PF), $\beta$-individual fairness ($\beta$-IF), the $(\gamma, \alpha)$-transferable core ($(\gamma, \alpha)$-TC), and the $\alpha$-$q$-core.
        See \cref{sec:bridges,sec:jrpjr} for the corresponding definitions and results.
        If $\textsc{a} \to \Pi$, then algorithm \textsc{a} produces outcomes satisfying $\Pi$.
        If $\Pi \to \Gamma$, then any outcome satisfying $\Pi$ also satisfies $\Gamma$.
        If $\Gamma$ takes a parameter $\alpha$, then the label specifies the parameter that can be satisfied (for the transferable core, the result holds for all $\gamma > 1$).
        \emph{Right:}
        The metric space for the examples used throughout the paper. Edges without labels have length $1$, the distance between any two points is given by the length of the shortest path between them.
    }
    \label{fig:results}
\end{figure}
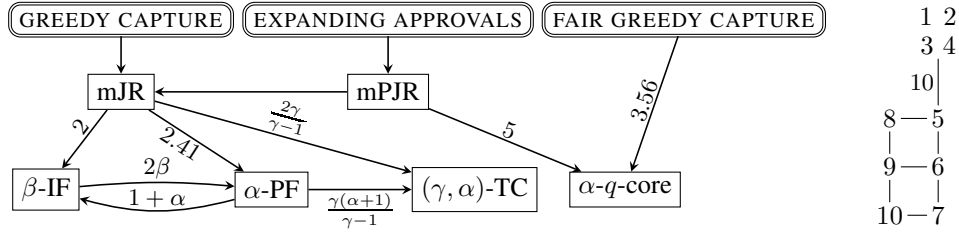

\paragraph{Related work.}
\emph{Individual fairness} was introduced by \citet{JKL20a}.
Since then, follow-up work mainly focused on bi-criteria approximation guarantees \citep{MaVa20a, NeCh21a, VaYa22b, CDCS22a, BCEL24a}. Additionally, \citet{HXXY23a} studied individual fairness for clustering with outliers and \citet{StCo23a} incorporated demographic information into individual fairness.
The individual fairness notion was also carried over to the setting of approval-based multiwinner voting \citep{BIMP22a}.
We mention that the name ``individual fairness'' is also used for other (unrelated) fairness notions \citep[e.g.][]{KKM+23a,CDE+22b}.

\emph{Proportional fairness} was first studied by \citet{CFL+19a}.
\citet{MiSh20a} showed that the \textsc{greedy capture} algorithm by \citeauthor{CFL+19a} achieves better approximation guarantees in certain metric spaces (including the Euclidean space with the $2$-norm) and studied its complexity.
\citet{LLS+21a} introduced notions inspired by \citeauthor{CFL+19a}, which are related to the transferable core concept from algorithmic game theory.
\citet{ALM23a} introduces proportionality axioms and rules directly inspired from social choice theory to proportional clustering.
Among other things, they show that every outcome satisfying DPRF (see \cref{ssec:jr-axioms}) achieves an $(1+\sqrt{2})$-approximation to proportional fairness.
Further connections to social choice or relations between the above fairness notions of \citet{JKL20a} or \citet{LLS+21a} remain unexplored, though.

\citet{EbMi23a} study proportionality in the setting of sortition (see e.g., \citet{FGG+21a}), proposing a generalization of proportional fairness and a refined variant of \textsc{greedy capture}.
This variant and its proportionality were used by \citet{CMP24a} to construct panels whose decisions align with that of the underlying population.
The most recent work directly related to ours was created independently and in parallel to ours by \citet{KKK24b}.
They study proportional fairness and the transferable core in an incomplete information setting and show that just knowing the order of the distances to between agents and center candidates suffices to achieve a 5.71-approximation to proportional fairness.

\emph{Multiwinner voting} is the branch of computational social choice theory dealing with selecting multiple instead of just one candidate as a winner.
A main branch herein focuses on \emph{proportionality}.
While much of the literature on proportionality, starting with \citet{ABC+16a}, focuses on approval preferences (see \citet{LaSk22a} for a recent book on this topic),
proportionality notions also exist for ordinal preferences \citep{Dumm84a}.
These notions were recently strengthened by \citet{AzLe20a, AzLe21a} and \citet{BrPe23a}, with the latter forming the basis for the proportionality axioms we discuss in this paper. 
We are further closely related to the works of \citet{CSV22a} and \citet{EKM+22a} who studied the representation of a given committee by investigating the distances of agents to their $q$-closest committee member.

\paragraph{Model and notation.}
Let $(\mathcal X, d)$ be a (pseudo)-metric space with a distance function $d \colon \mathcal X \times \mathcal X \to \mathbb R$ satisfying $d(i, i) = 0$, $d(i, j) = d(j, i)$ and $d(i, j) + d(j, k) \ge d(i,k)$.
Let $i \in \mathcal X$ be a point.
For $r \in \mathbb R$, define $B(i, r) = \{j \in \mathcal X \colon d(i,j) \le r\}$ to be the ball of radius $r$ around $i$.
For $W \subseteq \mathcal X$, let $d(i, W) = \min_{c \in W} d(i, c)$. %
For $q \le |W|$, $d^q(i, W)$ is distance to the $q$-th closest point in $W$ to $i$. %
Note that $d^1(i, W) = d(i, W)$ and that $d^q(i, W) \le d(i, j) + d^q(j, W)$ for $i, j \in N$.

Throughout the paper, we are given a set of \emph{agents} $N = [n]$ and a (possibly infinite) set of \emph{candidates (facilities)} $C$,
both of which lie in a metric space $(\mathcal X, d)$, and a number $k \in \mathbb N^+$.
A \emph{clustering} or \emph{outcome} is a subset $W \subseteq C$ of at most $k$ candidates.
The elements $c \in W$ are called \emph{centers}.
Our examples use the \emph{(weighted) graph metric}
in which the points are the vertices of a graph with edge lengths, and the distance between two points is the length of a shortest path between them.

\section{Relations between proportional fairness notions}
\label{sec:bridges}
In this section, we prove the relations between proportional fairness \citep{CFL+19a}, individual fairness \citep{JKL20a}, and the transferable core \citep{LLS+21a}.
We first define the notions.

The idea of \emph{proportional fairness} is the following:
If there is a candidate $c$ such that at least $\frac{n}{k}$ agents are closer to $c$ by a factor $\alpha$ than to their closest cluster center in the outcome $W$, then we say that the agents will \emph{deviate} to $c$.
If there is no such candidate, the outcome satisfies $\alpha$-proportional fairness.
\begin{definition}
	\label{def:PF}
	For $\alpha \ge 1$ an outcome $W$ satisfies \emph{$\alpha$-proportional} fairness, if there is no group $N' \subseteq N$ of agents with $\lvert N' \rvert \ge \frac{n}{k}$ and $c \notin W$ such that $\alpha \cdot d(i, c) < d(i, W)$ for all $i \in N'$.
\end{definition}
While $(2 - \varepsilon)$-proportional fair outcomes need not exist (for any $\varepsilon > 0$), $(1 + \sqrt{2})$-proportional fair outcomes can be computed for any metric space \citep{CFL+19a, MiSh20a}.

To define \emph{individual fairness},
denote by $r_{N,k}(i)$ be the radius of the smallest ball around an agent $i \in N$ that encloses at least $\frac{n}{k}$ agents,
i.e., $r_{N,k}(i) = \min\{r \in \mathbb R \colon |B(i, r) \cap N| \ge \frac{n}{k}\}$.
We drop the subscripts $N$ and $k$ if clear from context.
For this definition to properly work, we additionally need the assumption that $N \subseteq C$, i.e., any agent can be chosen as center.
Otherwise, a secluded group of agents without any possible cluster centers around them would never be able to get a center close to them in the outcome.
Indeed, this is a plausible restriction in metric clustering, as oftentimes the centers may be picked from the (infinite) set of points in the metric space.%
\begin{definition}
	\label{def:IF}
	For an instance with $N \subseteq C$, for $\beta \ge 1$
	an outcome $W$ satisfies \emph{$\beta$-individual fairness} if $d(i, W) \le \beta r_{N,k}(i)$ for all $i \in N$.
\end{definition}
It is known that an outcome satisfying $2$-individual fairness always exists, while there are instances with no $(2 - \varepsilon)$-individually fair outcome \citep{JKL20a}. 

The \emph{transferable core}\footnote{We remark that \citet{LLS+21a} call this notion just ``core''.} notion is based on the concept of transferable utilities from game theory.
Comparing to proportional fairness, the notion considers the average utility for each group.

\begin{definition}
	\label{def:TC}
	For $\gamma, \alpha \ge 1$, an outcome $W$ is in the \emph{$(\gamma, \alpha)$-transferable core} if there is no group of agents $N' \subseteq N$ and candidate $c \notin W$ with $\lvert N' \rvert \ge \gamma \frac{n}{k}$ and
	$ %
	\alpha \sum_{i \in N'} d(i, c) < \sum_{i \in N'} d(i, W).
	$ %
\end{definition}
It is known that the for any $\gamma > 1$ there are outcomes in the $\smash{(\gamma, \max(4, \frac{3\gamma - 1}{\gamma - 1}))}$-transferable core while there need not be outcomes in the $\smash{(\gamma, \min(1, \frac{1}{\gamma - 1}))}$-transferable core \citep{LLS+21a}.
\newcommand{\exDefPFIFTC}{
\begin{example}
    Consider the instance depicted in \cref{fig:results} (right) with $k = 5$ and the associated graph distance metric.
    Assume that cluster centers can only be placed on the depicted agents.
    We have $\frac{n}{k} = 2$; thus any two agents are able to deviate to another center.
    The outcome $W = \{1,2,3,6,9\}$ satisfies $1$-proportional fairness:
    The agents $1, \dots, 4$ have distance $0$ to a center, while every remaining agent has distance at most $1$ to a center.

    To see the difference between proportional fairness, individual fairness, and the transferable core, consider the same instance with $k = 4$, so $\frac{n}{k} = 2.5$.
    Here, the outcome $W = \{1,2,6,7\}$ satisfies $1$-proportional fairness, however it does not satisfy $1$-individual fairness.
    Agent $8$ could look at their $2$ closest neighbors, $5$ and $9$, both at a distance of $1$.
    However, the distance of $8$ to the outcome is $2$.
    Observe that $W$ also is not in the $(1,1)$-transferable core.
    Here, for the group $N' = \{8, 9, 10\}$ and candidate $c = 9$, we have $\sum_{i \in N'} d(i, c) = 2 < \sum_{i \in N'} d(i, W) =4$.
\end{example}
} \exDefPFIFTC

\subsection{Proportional and individual fairness}
We first show that proportional and individual fairness are the same up to a factor of at most $2$.
\begin{theorem}[
    label=thm:IF-PF,
    restate=IFPF]
    Let $\alpha, \beta \ge 1$. If $N \subseteq C$, then
    an outcome that satisfies $\alpha$-proportional fairness also satisfies $(1 + \alpha)$-individual fairness,
    and an outcome that satisfies $\beta$-individual fairness also satisfies $2\beta$-proportional fairness.
    If $N = C$, then
    an outcome that satisfies $\beta$-individual fairness also satisfies $(1 + \beta)$-proportional fairness.
\end{theorem}
\newcommand{\proofIFPF}{
\begin{proof}
	Let $W \subseteq C$ be an outcome satisfying $\alpha$-proportional fairness, $j \in N$ be any agent, and $N_j = \{ i \in N \colon d(i,j) \le r(j)\}$.
	As $N \subseteq C$, there is an $i \in N_j$ with $d(i, W) \le \alpha d(i,j)$; otherwise the coalition $N_j$ deviates to candidate $j$.
	Thus, by the triangle inequality,
	\(
		d(j, W) \le d(i, j) + d(i, W) \le (1+\alpha) d(i, j) \le (1+\alpha) r(j),
	\)
	and $W$ satisfies $(1+\alpha)$-individual fairness.

	Now suppose the outcome $W$ satisfies $\beta$-individual fairness.
	Let $N' \subseteq N$ with $|N'| \ge \frac{n}{k}$ and $c \notin W$ be an unchosen candidate.
	Take $i^* \in N'$ to be the agent in $N'$ furthest away from $c$.
    If $N \subseteq C$, then
	the radius $r(i^*)$ containing $\lceil\frac{n}{k}\rceil$ agents is at most as large as the most distant agent in $N'$,
	i.e., there is an $i' \in N'$ with $r(i^*) \le d(i^*, i') \le d(i^*, c) + d(c, i')$.
	Then
	\(
		d(i^*, W) \le \beta r(i^*) \le \beta (d(i^*, c) + d(c, i')) \le 2 \beta d(i^*, c).
	\)
    If $N=C$,
    then, since $|N'| \ge \frac{n}{k}$, we have $r(c) \le d(c, i^*)$; thus $d(c, W) \le \beta d(c, i^*)$.
    Therefore,
	$ %
		d(i^*, W) \le d(i^*, c) + d(c, W) \le (1+\beta) d(i^*, c).
	$ %
\end{proof}
} %
\proofIFPF

Indeed, all three provided bounds are tight.

\begin{theorem}[
	label=thm:if-pf-tightness,
	restate=IFPFtight
	]
    For every $\alpha, \beta \ge 1$ and $\varepsilon > 0$,
    there are instances with $N = C$ for which there exists
    \begin{inparaenum}[(1)]
        \item an outcome which satisfies $\alpha$-proportional fairness, but not $(1 + \alpha - \varepsilon$)-individual fairness, and
        \item an outcome which satisfies $\beta$-individual fairness, but not $(1 + \beta - \varepsilon$)-proportional fairness.
        Moreover, there are instances with $N \subseteq C$ for which there exists
        \item an outcome which satisfies $\beta$-individual fairness, but not $(2\beta - \varepsilon)$-proportional fairness.
    \end{inparaenum}
\end{theorem}
\newcommand{\proofIFPFtight}{
\begin{proof}
    The constructed instances are based on points under a graph metric and are depicted in \cref{fig:if-pf-tightness}.
    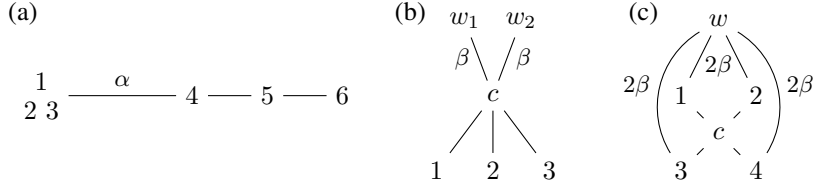
\begin{figure}
        \centering
        \begin{tikzpicture}
            \begin{scope}
                \node at (-0.25,1.6) {(a)};
                \node [align=center] (v1) at (0,0.5) {$1$\\$2\,\,3$};
                
                \node (v4) at (2,0.5) {$4$};
                
                \node (v5) at (3,0.5) {$5$};
                
                \node (v6) at (4,0.5) {$6$};
                
                \draw (v1) -- node[above] {\footnotesize $\alpha$} (v4);
                \draw (v4) -- (v5);
                \draw (v5) -- (v6);
            \end{scope}
            \begin{scope}[xshift=5.5cm]
                \node at (-0.6,1.6) {(b)};

                \node (w1) at (0.125,1.5) {$w_1$};
                \node (w2) at (0.875,1.5) {$w_2$};

                \node (c) at (0.5,0.5) {$c$};
                \node (v1) at (-0.25,-0.5) {$1$};
                \node (v2) at (0.5,-0.5) {$2$};
                \node (v3) at (1.25,-0.5) {$3$};

                \draw
                    (c) -- (v1)
                    (c) -- (v2)
                    (c) -- (v3);

                \draw (c) to node[left] {\footnotesize $\beta$} (w1);
                \draw (c) to node[right] {\footnotesize $\beta$} (w2);
            \end{scope}
            \begin{scope}[xshift=8.5cm]
                \node at (-0.5,1.6) {(c)};
                \node (v1) at (0,0.5) {$1$};
                \node (v2) at (1,0.5) {$2$};
                \node (v3) at (0,-.5) {$3$};
                \node (v4) at (1,-.5) {$4$};
                \node (w)  at (0.5,1.5) {$w$};
                \node (c)  at (0.5,0) {$c$};

                \node at (0.5,0.9) {\footnotesize $2\beta$};

                \draw
                    (c) -- (v1) -- (w)
                    (c) -- (v2) -- (w)
                    (c) -- (v3)
                    (c) -- (v4);

                \draw[bend left=45]
                    (v3) to node[left] {\footnotesize $2\beta$} (w);
                \draw[bend right=45]
                    (v4) to[bend right] node[right] {\footnotesize $2\beta$} (w);
            \end{scope}
        \end{tikzpicture}
        \caption{Example metric spaces for the lower bounds in \Cref{thm:if-pf-tightness}. Edges without labels have length $1$.}
        \label{fig:if-pf-tightness}
    \end{figure}

    First, for the example for proportional fairness and $N=C$, consider the instance (a) with $N$ consisting of 6 agents, and $k = 2$.
    There the outcome $\{2, 3\}$ satisfies $\alpha$-proportional fairness, since the only deviating coalition could be $4, 5, 6$ and by deviating to $5$ they could at most improve by a factor of $\alpha$, if they would deviate to $4$ it would be at most by a factor of $\frac{\alpha + 1}{2} \le \alpha$. Further, due to $5$ the instance is also only $\alpha + 1$-individual fair. 

    For the example for individual fairness and $N=C$, consider the instance (b) with 6 agents, and $k=2$.
    The outcome $\{w_1, w_2\}$ satisfies $\beta$-individual fairness as $d(c, W) \le \beta r(c)$.
    Note that $d(i, W) = 1+\beta$ while $r(i) = 2$ for each $i \in \{1, 2, 3\}$.
    However, for each $i \in \{1, 2, 3\}$, we have $d(i, c) = 1 \not < \frac{1}{1+\beta}d(i, W)$; thus the outcome does not satisfy $(1+\beta - \varepsilon)$-proportional fairness.

    For the example for individual fairness and $N \subseteq C$, consider the instance (c) with agent set $N = \{1, 2, 3, 4\}$, candidate set $C = N \cup \{c, w\}$, and $k=1$.
    There, the outcome $\{w\}$ satisfies $\beta$-individual fairness as for each $i \in N$, the smallest radius containing $\frac{n}{k}=4$ agents is $2$, and the distance to $w$ is $2\beta$.
    However, for each $i \in N$, $d(i, c) = 1 \not < \frac{1}{2\beta}d(i, W)$; thus it does not satisfy $(2\beta-\varepsilon)$-proportional fairness.
\end{proof}
} \proofIFPFtight

\subsection{Proportional fairness and the transferable core}

It is easy to see that the $(1, \alpha)$-transferable core implies $\alpha$-proportional fairness.
For $\gamma > 1$ however, the $(\gamma, \alpha)$-transferable core does not imply any meaningful proportional fairness approximation (consider $\frac{n}{k}$ agents on one point and $(\gamma - 1) \frac{n}{k}$ agents ``far'' away).
Hence, we focus on the other direction and show that a proportional fairness approximation implies one to the transferable core.
\begin{theorem}[
	label=thm:PF-TC,
	restate=PFTC]
    An outcome satisfying $\alpha$-proportional fairness is in the $(\gamma, \frac{\gamma(\alpha+1)}{\gamma-1})$-transferable core for any $\alpha\ge 1$ and $\gamma > 1$. 
\end{theorem}
\newcommand{\proofPFTC}{
\begin{proof}
    Let $W \subseteq C$ satisfy $\alpha$-proportional fairness.
    Let $N' \subseteq N$ be a group of size $n' \ge \gamma \frac{n}{k}$, $c \notin W$, and shorten $\eta = \lceil \frac{n}{k} \rceil$.
    Further, the agents $N' = \{i_1, \dots, i_{n'}\}$ are ordered by their increasing distance to $c$, i.e., $d(i_j, c) \le d(i_{j+1}, c)$ for every $j \in [n'-1]$.
    Let $J_0 = \{i_1, \dots, i_\eta\}$
    and $j_0 \in J_0$ such that $d(j_0, W) \le \alpha d(j_0, c)$; such an agent must exist due to $\alpha$-proportional fairness.
    Next, for $\lambda = 1, \dots, n'-\eta$, we inductively define $J_\lambda = \{i_1, \dots, i_{\eta+\lambda}\} \setminus \{j_0, \dots, j_{\lambda-1}\}$,
    and choose $j_\lambda \in J_\lambda$ such that $d(j_\lambda, W) \le \alpha d(j_\lambda, c) \le \alpha d(i_{\eta + \lambda}, c) $ (note that $|J_\lambda| = \eta$).
    Thus, 
    \begin{equation}
	    \label{eq:jlambda}
	    \textstyle \sum_{\lambda = 0}^{n'-\eta} d(j_\lambda, W) \le \alpha \sum_{z=\eta}^{n'} d(i_z, c) \le \alpha \sum_{i \in N'} d(i, c).
    \end{equation}
    Next, for each $i \in N'' = N' \setminus \{j_0, \dots, j_{n'-\eta}\}$,
    we can bound the distance to $W$ as follows:
    \begin{equation}
        \label{eq:TC1}
	\begin{aligned}
		d(i, W) &\le d(i, c) + d(c, j_0) + d(j_0, W) \le d(i, c) + (1+\alpha) d(i_\eta, c).
	\end{aligned}
    \end{equation}
    Note that $d(i_\eta, c) \le \frac{1}{n'-|N''|}\sum_{z=\eta}^{n'} d(i_z, c)$
    as each of the summands is at least $d(i_\eta, c)$.
    Thus,
    \begin{equation}
	    \label{eq:iz}
	    \begin{aligned}
	    &\textstyle\sum\limits_{i \in N''}\! d(i, W) \le \!\!\!\sum\limits_{i \in N''}\!\! \big( d(i, c) + (1+\alpha)d(i_\eta, c) \big)
		    \le{} \textstyle (1+\alpha)\frac{|N''|}{n'-|N''|} \!\sum\limits_{z=\eta}^{n'}\! d(i_z, c) + \!\!\!\sum\limits_{i \in N''}\!\! d(i, c).
	    \end{aligned}
    \end{equation}
    As $n' \ge \gamma\frac{n}{k}$ and $|N''| = \eta-1 \le \frac{n}{k}$, we have $\frac{|N''|}{n'-|N''|} \le \frac{1}{\gamma - 1}$.
    In all, $\sum_{i \in N'} d(i, W)$ is the sum of \eqref{eq:jlambda} and \eqref{eq:iz}, which is at most
    \begin{equation}
        \label{eq:TC2}
        \textstyle \big(\frac{\alpha+1}{\gamma-1} + \alpha + 1\big) \sum\limits_{i \in N'} d(i, c) = \frac{\gamma(\alpha+1)}{\gamma-1} \sum\limits_{i \in N'} d(i, c).
    \end{equation}
    Thus, $W$ is in the $(\gamma, \frac{\gamma(\alpha+1)}{\gamma-1})$-transferable core.
\end{proof}
} \proofPFTC

We remark that the denominator $\frac{1}{\gamma-1}$ in $\alpha$ is inevitable. This is because for $\gamma \le 2$, the $(\gamma, \frac{1}{\gamma-1})$-transferable core may be non-empty \cite[Theorem~18]{LLS+21a}.
We complement the above upper bound with an asymptotically tight lower bound. 

\begin{theorem}[
	label=thm:lb:tq,
	restate=lbtq
	]
    For any $\alpha \ge 1$, $\gamma > 1$, and $\varepsilon > 0$ there exists an instance in which an $\alpha$-proportional fair outcome is not in the $(\gamma, \frac{\gamma \alpha +1}{\gamma - 1} - \varepsilon$)- transferable core.
\end{theorem}
\newcommand{\prooflbtq}{
\begin{proof}
    Consider an instance with a candidate $c$ in which there is a set $N_0$ of $\lceil \frac{n}{k} - 1\rceil$ agents at distance $0$ to $c$, while the rest of the agents are at distance $1$ to $c$.
    Our outcome $W$ contains one cluster center $c_1$ at distance $\alpha$ to every agent \emph{not} in $N_0$.
    It is easy to see that this is indeed $\alpha$-proportional, the agents could only deviate to $c$ and improve by at most a factor of $\alpha$.
    Further, we get that $\sum_{i \in N'} d(i, W) = (\lceil \frac{n}{k} - 1\rceil) (\alpha + 1) + (\gamma \frac{n}{k} - \lceil \frac{n}{k} - 1\rceil)\alpha$ while $\sum_{i \in N'} d(i,c) = (\gamma \frac{n}{k} - \lceil \frac{n}{k} - 1\rceil)$.
    Thus, in this instance 
    \begin{align*}
        \frac{\sum_{i \in N'} d(i, W)}{\sum_{i \in N'} d(i,c)} &= \frac{(\lceil \frac{n}{k} - 1\rceil) (\alpha + 1) + (\gamma \frac{n}{k} - \lceil \frac{n}{k} - 1\rceil)\alpha}{\gamma \frac{n}{k} - \lceil \frac{n}{k} - 1\rceil} \\
                                                               &= \alpha + \frac{(\lceil \frac{n}{k} - 1\rceil) (\alpha + 1)}{\gamma \frac{n}{k} - \lceil \frac{n}{k} - 1\rceil}
                                                               \ge \alpha + \frac{( \frac{n}{k} - 1) (\alpha + 1)}{\gamma \frac{n}{k} -  \frac{n}{k} + 1}
                                                               \ge \alpha + \frac{(\alpha + 1)(\frac{n}{k}) - (\alpha + 1) }{(\gamma - 1) (\frac{n}{k}) + 1}\\
                                                               &= \alpha + \frac{(\alpha + 1) - \frac{k}{n}(\alpha + 1) }{(\gamma - 1) + \frac{k}{n}}.
    \end{align*}
     As $\frac{k}{n}$ goes to $0$ this goes to $\alpha + \frac{\alpha + 1}{\gamma - 1} = \frac{\gamma \alpha +1}{\gamma - 1}$.
\end{proof}
} \prooflbtq

\section{Fairness notions multiwinner voting axioms}
\label{sec:jrpjr}
In this section we show a connection between the research on computational social choice, specifically approval-based multiwinner voting (also known as approval-based committe (ABC) voting) and the fairness notions for clustering.
We will first give a primer on ABC voting and introduce our \emph{metric JR axioms}.
We then focus on two of those axioms
and show that
(1) they are satisfied by existing, simple algorithms, and
(2) they imply the best known approximation guarantees to proportional and individual fairness, the transferable core, and the $q$-core (see \cref{def:core} below).
For the latter two notions, we are even able to improve upon the best currently known approximation guarantees.
Finally, we will focus on the case when the candidate set is infinitely large (i.e., when we are in the Euclidean space and every point is a candidate):
In this setting, the above algorithms become hard to compute.
We combine two approaches to maneuver around this hardness and again match upon the best known approximation guarantees for proportional fairness and the $q$-core.

\subsection{Metric JR axioms}
\label{ssec:jr-axioms}
In ABC voting \citep{LaSk22a}, we are given a set $N$ of voters (or agents), a set $C$ of candidates, and a committee size $k$.
For each voter $i \in N$ we are given a subset $A_i \subseteq C$ of candidates they approve.
For such preferences, we call a set $N' \subseteq N$ of voters
\emph{$\ell$-large} if $|N'| \ge \ell \frac{n}{k}$, and
\emph{$\ell$-cohesive} if $\lvert \bigcap_{i \in N'} A_i \rvert \ge \ell$.
We say that a committee satisfies
\begin{compactitem}
    \item[JR] if for every $1$-cohesive and $1$-large group $N'$ there exists an $i \in N'$ with $\lvert A_i \cap W \rvert \ge 1$;
    \item[PJR] if for every $\ell \in [k]$ and $\ell$-cohesive and $\ell$-large group $N'$ it holds that $\lvert \bigcup_{i \in N'} A_i \cap W \rvert \ge \ell$,
\end{compactitem}
and remark that there are many further proportionality axioms \citep{LaSk22a}.
Here, JR is short for \emph{justified representation}.
To define our \emph{metric JR axioms} for voters and candidates in a distance metric, we follow \citet{BrPe23a} (who lifted these axioms for weak ordinal preferences and called them rank-$\Pi$) and generalize their notions to look at each distance separately.
\begin{definition}[Metric JR axioms]
	\label{def:rank-pi}
	Let $(\mathcal X, d)$ be a distance metric, let $N, C \subseteq \mathcal X$.
	Let $\Pi$ be a proportionality axiom.
	An outcome $W$ satisfies \emph{m$\Pi$} (short for metric) if for all $y \in \mathbb R_{\ge 0}$,
	for the ABC voting instance in which each $i \in N$ has the approval set $B(i, y) \cap C$,
	the outcome $W$ satisfies $\Pi$.
\end{definition}

For example,
an outcome $W$ satisfies \rjr if for every $y \in \mathbb R$ and for every group $N' \subseteq N$ of at least $\frac{n}{k}$ agents whose ball of radius $y$ all contain a common candidate ($|\bigcap_{i \in N'} B(i, y) \cap C| \ge 1$), there exists an agent $i \in N'$ whose ball of radius $y$ contains a center $c \in W$.

We want to point out that \rjr is \emph{significantly} weaker than \rpjr.
Indeed, to satisfy \rpjr, an outcome $W$ may need to contain several candidates $c$ such that $d(i, c) > d(i, W)$ for all agents $i \in N$, i.e., $c$ is no-ones ``first choice'' among $W$; \rjr does not have this property.
This makes \rjr the more sensible of the two axioms for ``vanilla'' clustering, in which one only cares about the closest center to each agent.
\rpjr however, is a natural axiomatic choice for settings such as sortition or even social choice in general: Here, agents may benefit from having more than a single representative.
We provide some intuition for \rjr and \rpjr and this property in the example below.

\newcommand{\exDefJR}{
\begin{example}
	\label{ex:jr}
	To see the differences between the proportionality axioms, consider the instance depicted in \Cref{fig:results} (right).
	First consider instance (a) on the left with $N=C=\{1, \dots, 10\}$, $k = 4$, and the outcome $W = \{1,2,3,6\}$.
	Here, $\frac{n}{k} = 2.5$.
	First, we note that this outcome does not satisfy $1$-proportional fairness:
	The agents $8, 9, 10$ are closer to $9$ than they are to the closest winner in $W$.
	It does however satisfy \rjr: Among every group of at least three agents that have a common candidate within distance $y$, there is one agent that has a cluster center $w \in W$ within distance $y$.
	For example, $8, 9, 10$ have candidate $9$ at distance $1$, and the distance of $9$ to the closest center is also $1$.
	This outcome does not satisfy \rpjr though, since the group $5, \dots, 10$ would deserve at least two candidates within distance $1$ in $W$. 
	An outcome satisfying \rpjr is $W = \{1, 2, 3, 9\}$.
	For $y = 0$, only the group $\{1, \dots, 4\}$ shares a candidate, but also have a center at distance $0$.

	If $k=5$, then, to satisfy \rpjr, an outcome must contain at least two of the four candidates.
	But there are outcomes satisfying \rjr that contain only one of $1, \dots, 4$.
	This property of \rpjr makes it suited for settings in which agents may want to be represented by multiple candidates, e.g., in political settings, in which the candidates end up possessing voting power to represent the agents.
\end{example}
} \exDefJR

Independently of \citet{BrPe23a}, \citet{ALM23a} introduced two notions they call \emph{Proportionally Representative Fairness}.
The first notion is called ``discrete'' (DPRF), and the second is called ``unconstrained'' (UPRF).
Indeed, DPRF is equivalent to \rpjr.
UPRF was introduced to tackle the case when the candidate space is unbounded.
We discuss how it relates to \rpjr and the other fairness notions in \cref{app:uprf}.

\citet{ALM23a} show that an outcome satisfying DPRF (\rpjr) also fulfills $(1+\sqrt{2})$-proportional fairness.
We show hereafter that this already holds for the (much weaker) \rjr axiom.

\subsection{Fairness bounds for \rjr outcomes}
\label{ssec:rjr}

We now prove the approximation guarantees implied by \rjr.
Indeed, the bound for the transferable core below improves upon the analysis of \citet{LLS+21a}.

\begin{theorem}[
	label=thm:JR-PF-IF-TC,
	restate=JRPFIFTC
	]
	Let $W$ be an outcome satisfying \rjr.
	Then it also satisfies $(1 + \sqrt{2})$-proportional fairness, $2$-individual fairness, and is in the $(\gamma, \frac{2\gamma}{\gamma-1})$-transferable core for any $\gamma > 1$.
\end{theorem}
\newcommand{\proofJRPFIFTC}{
\begin{proof}
    For proportional fairness, our proof follows the lines of \citet[Theorem~1]{CFL+19a}.
    Let $N' \subseteq N$ be a group of size at least $\frac{n}{k}$ and let $c \notin W$ such that $d(i', c) \le d(i', W)$ for all $i' \in N'$. Let $i_1 = \argmax_{i \in N'} d(i, c)$ and let $y = d(i_1, c)$.
    Then $c \in \bigcap_{i \in N'} B(i, y) \cap C$,
    and \rjr implies the existence of an agent $i_2 \in N'$ and a candidate $c' \in W$
    such that $c' \in B(i_2, y) \cap W$.
    Then we have $d(i_2, W) \le d(i_2, c') \le d(i_1, c)$.
    Further, $d(i_1, W) \le d(i_1, c) + d(c, i_2) + d(i_2, W) \le 2 d(i_1, c) + d(i_2, c)$.
    Thus, we can apply \cref{lem:q-core-general} with $q=1$, $C' = \{c\}$, $\rho_1 = 2$, $\rho_2 = 2$, $\sigma_1 = 1$, and $\sigma_2 = 0$, which proves that $W$ is in the $(1+\sqrt{2})$-$1$-core; this is equivalent to satisfying $(1+\sqrt{2})$-proportional fairness.

    For individual fairness, we again assume that $N \subseteq C$.
    Let $i \in N$ be any agent and let $N' \subseteq N$ be a group of at least $\frac{n}{k}$ agents closest to $i$ (including $i$).
    Let $y = \max_{i' \in N'} d(i, i')$.
    Then $i \in \bigcap_{i' \in N'} B(i, y) \cap C$,
    and \rjr implies the existence of an agent $j \in N'$ and a candidate $c' \in W$
    such that $c' \in B(i, y) \cap W$.
    By the triangle inequality, $d(i, W) \le d(i, j) + d(j, W) \le 2 \max_{i' \in N'} d(i, i')$;
    thus $W$ satisfies $2$-individual fairness.

    For the transferable core, let $N' \subseteq N$ be a group of size $n' \ge \gamma \frac{n}{k}$ and let $c \notin W$.
    Let $\eta = \lceil \frac{n}{k} \rceil$.
    We can now use the same argumentation as in the proof of \Cref{thm:PF-TC}:
    Again, we order the agents $i_1, \dots, i_{n'}$ in $N'$ by their increasing distance to $c$.
    We choose the elements $j_\lambda \in J_\lambda$, $\lambda = 0, \dots, n' - \eta$, such that $d(j_\lambda, W) \le d(i_{\eta+\lambda}, c)$;
    such a $j_\lambda$ always exists by \rjr, as $|J_\lambda| = \eta$.
    By replacing every occurrence of $\alpha$ by $1$ in the inequality in the proof of \Cref{thm:PF-TC}, we obtain the claimed bound.
\end{proof}
}
\proofJRPFIFTC

If the candidate space is finite, then an outcome satisfying \rjr can be computed in polynomial time by the \textsc{greedy capture} algorithm \citep{CFL+19a,MiSh20a,LLS+21a}.
We briefly recall its procedure:

\textsc{greedy capture} starts off with an empty clustering $W$.
It maintains a radius $\delta$ (initially $\delta=0$) and smoothly increases $\delta$.
If there is a candidate $c$ such that at least $\frac{n}{k}$ agents have distance at most $\delta$ to $c$, it adds $c$ to $W$ and deletes the $\frac{n}{k}$ agents.
If an agent has distance at most $\delta$ to a candidate in $W$, then it is deleted as well.
This is continued until all agents are deleted.

\newcommand{\exGC}{
\begin{example}
    \label{ex:gc}
    Consider the instance in \Cref{fig:example-pr}.
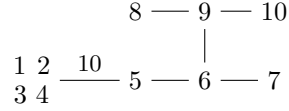
\begin{figure}
	\centering
	\begin{tikzpicture}[scale=0.92]
		\begin{scope}[yscale=-1]
			\node [align=center] (v1) at (.5,0.5) {$1\,\ 2$\\$3\,\,4$};

			\node (v4) at (2,0.5) {$5$};

			\node (v5) at (3,0.5) {$6$};

			\node (v6) at (4,0.5) {$7$};

			\node (v7) at (2,-0.5) {$8$};

			\node (v8) at (3,-0.5) {$9$};

			\node (v9) at (4,-0.5) {$10$};

			\draw (v1) -- node[above] {\footnotesize $10$} (v4);
			\draw (v4) -- (v5);
			\draw (v5) -- (v6);
			\draw (v5) -- (v8);
			\draw (v7) -- (v8);
			\draw (v8) -- (v9);
		\end{scope}
	\end{tikzpicture}
	\caption{Metric space for some of the examples. Edges without labels have length $1$.}
	\label{fig:example-pr}
\end{figure}
    Here, with $k = 4$, \textsc{greedy capture}, would first open one of $1,\dots, 4$ with $\delta = 0$ and remove all agents from $1, \dots, 4$.
    Then for $\delta = 1$ it could either open $6$ or $9$, removing all adjacent agents to it.
    Then there are two agents remaining, which would be assigned to either $6$ or $9$ for $\delta = 2$.
    Thus, in this instance, \textsc{greedy capture} only opens two clusters. 
\end{example}
} \exGC

\begin{proposition}[
	label=prop:gc-jr,
	restate=algsJR
	]
    Any outcome returned by \textnormal{\textsc{greedy capture}} satisfies \rjr.
\end{proposition}
\newcommand{\proofalgsJR}{
\begin{proof}
    Let $r \ge 0$. Then, after \textsc{greedy capture} reached $\delta = r$ there is no unchosen candidate $c \notin W$ with at least $\frac{n}{k}$ agents at a distance of at most $r$. Therefore, \textsc{greedy capture} satisfies \rjr. 
\end{proof}
} \proofalgsJR

\subsection{Fairness bounds for \rpjr outcomes}
\label{ssec:rpjr}

Recall that \rpjr is equivalent to DPRF.
Thus, as shown by \citet{ALM23a}, an outcome satisfying \rpjr always exists and can be computed in polynomial time for finite candidate spaces by the \textsc{expanding approvals} algorithm, which is also established in multiwinner voting \citep{AzLe20a, AzLe21a}.
The algorithm behaves similarly to \textsc{greedy capture}.
It starts off with an empty clustering $W$ and a radius $\delta=0$ as well
and additionally
gives each agent a budget $b_i = \frac{k}{n}$.
It then smoothly increases the radius $\delta$.
When there is a candidate $c \notin W$ for which the agents at a distance of at most $\delta$ have a budget of at least $1$, it decreases the budget of these agents collectively by exactly $1$ and adds $c$ to $W$.

\newcommand{\exEA}{
\begin{example}
	\label{ex:ea}
    Consider the instance in \Cref{fig:example-pr}.
    Here, with $k = 4$, \textsc{expanding approvals}, would give each agent a budget of $\frac{4}{10} = \frac{2}{5}$.
    For $\delta = 0$, it will open a cluster from $1, \dots, 4$ and decrease their budgets by exactly $1$, for instance it could set the budget of $1$ and $2$ to $0$ and of $3$ to $\frac{1}{5}$.
    Then for $\delta = 1$ it could again open $6$ and $9$, for instance by removing the budget of $6$ and $9$ to $\frac{1}{5}$ and of $5,7,8,10$ to zero.
    The remaining budget is exactly $1$, which would be spent for $\delta = 10$ on $5$.
    Thus, one possible final outcome is $\{1,5,6,9\}$.
\end{example}
} \exEA

Remarkably, it does not matter in which way the budget is subtracted and which candidate meeting the budget is selected; the outcomes of the algorithm will satisfy \rpjr in any~case.
\begin{proposition}[
	label=prop:ea-pjr,
	restate=algsPJR
	]
    Any outcome returned by \textnormal{\textsc{expanding approvals}} satisfies \rpjr.
\end{proposition}
\newcommand{\proofalgsPJR}{
\begin{proof}
	Let $r \ge 0$. Then, after \textsc{expanding approvals} reached $\delta = r$, the agents at a distance at most $r$ to a candidate $c$ must have a budget of less than $1$. Thus, if there are $\ell \frac{n}{k}$ agents at a distance at most $r$, they had a total budget of $\ell$ at the start of the algorithm and must have thus spent their money on at least $\ell$ candidate who are at a distance of at most $r$ to an agent in that group. Thus, \textsc{expanding approvals} satisfies \rpjr.
\end{proof}
} \proofalgsPJR
We now turn to fairness measures implied by \rpjr.
As any outcome satisfying \rpjr also fulfills \rjr, the results in \cref{thm:JR-PF-IF-TC} also hold for \rpjr.
Indeed, \rpjr is stricter in the sense that larger groups must also be represented justly by a proportional number of candidates:
an $\alpha$ percentage of the population should roughly be close to an $\alpha$ percentage of the centers.

This property makes \rpjr fit well into metric social choice settings such as sortition \citep{EbMi23a}.
They introduced a fairness notion that measures proportionality in this setting by considering for each agent not only the closest center, but the first $q$ closest centers.
In that, their notion called $\alpha$-$q$-core naturally generalizes $\alpha$-proportional fairness; the two are equal when $q=1$.

\begin{definition}
    \label{def:core}
    For $\alpha \ge 1$ an outcome $W$ is in the \emph{$\alpha$-$q$-core}, if there is no $\ell \in \mathbb{N}$ and no $N' \subseteq N$ with $|N'| \ge \ell \frac{n}{k}$ and set $C' \subseteq C$ with $q \le \lvert C' \rvert \le \ell$ such that $\alpha \cdot d^q(i, C') < d^q(i, W)$ for all $i \in N'$.
\end{definition}
\newcommand{\exDefqCore}{
\begin{example}
    \label{ex:qcore}
    Consider the instance in \cref{fig:example-pr} with $k=5$ and the outcome $W = \{1, 2, 3, 6, 9\}$.
    As mentioned above, $W$ satisfies $1$-proportional fairness.
    For the $3$-core however, consider the set $N' = \{5, \dots, 10\}$ deviating to $C' = \{6, 9, 10\}$.
    The distance of any member of $N'$ to $6$, $9$, or $10$ is at most $3$, while the distance to the third most distant center in the outcome is at least $10$.
    Thus, when considering the distances to the third most distant candidate in $C'$, every agent in $N'$ would improve by at least a factor of $\frac{10}{3}$.
    Thus, $W$ is only in the $\frac{10}{3}$-$3$-core.
\end{example}
} \exDefqCore

We mention in passing that we introduce similar gerneralizations for individual fairness and the transferable core, in which each agent is represented by $q$ candidates instead of one.
The definitions and obtained results can be found in \cref{app:q:gem}.

\citet{EbMi23a} show that, if $N=C$ (every agent is a candidate and vice versa), for a given $q$, one can compute a $\frac{5 + \sqrt{41}}{2}$-$q$-core outcome.\footnote{In addition, their randomized algorithm selects each agent with the same probability, a desirable property in the context of sortition, i.e., the randomized selection of citizen assemblies \citep{FGG+21a}.}
We show that \rpjr (or DPRF) provides a better guarantee for the $q$-core, for all values of $q$ simultaneously. 

\begin{theorem}[label=thm:rPJR,restate=rPJR]
    If an outcome satisfies \rpjr,
    then, for every $q \le k$, it is in the $5$-$q$-core.
\end{theorem}

To prove the theorem, we use two lemmas.
The first was first observed by \citet{EbMi23a} and is proven here in a shorter fashion.
\begin{lemma}[label=lem:pigeon,restate=pigeon]
	Let $\ell \ge q \ge 0$, let $N' \subseteq N$ be a set of agents with $N' \ge \ell\frac{n}{k}$,
	and let $C'$ be a set of $q \le |C'| \le \ell$ agents such that $d^q(i, C') \le d^q(i, W)$ for any $i \in N'$.
	Then there is a set $N''$ of at least $q \frac{n}{k}$ agents and a candidate $c \in C'$ such that $d(i, c) \le d^q(i, C')$ for all $i \in N''$. 
\end{lemma}
\newcommand{\proofpigeon}{
\begin{proof}
    Assume that each voter marks each of their top $q$ choices among $C'$.
    Then there are at least $q\lvert C'\rvert\frac{n}{k}$ marks on the candidates.
    Thus, there is one $c \in C'$ with at least $q\frac{n}{k}$ marks.
\end{proof}
} \proofpigeon

The next lemma bounds the $\alpha$-$q$-core once we find two agents with specific bounds on their distances.

\begin{lemma}[
	label=lem:q-core-general,
	restate=qCoreGeneral
	]
	Let $\rho_1, \rho_2, \sigma_1, \sigma_2 \ge 0$ and
	let $W \subseteq C$ be an outcome.
	If for any set $N' \subseteq N$ of at least $\ell\frac{n}{k}$ agents
	and any candidate set $C' \subseteq C$ with $q \le |C'| \le \ell$
	there are $i_1, i_2 \in N'$ such that
	$d^q(i_1, W) \le \rho_1 d^q(i_1, C') + \rho_2 d^q(i_2, C')$ and  $d^q(i_2, W) \le \sigma_1 d^q(i_1, C') + \sigma_2 d^q(i_2, C')$,
	then $W$ is in the $\alpha$-$q$-core, where $\alpha \le \rho_1 + \frac{1}{2}\big(\sigma_2 - \rho_1 + \sqrt{(\rho_1-\sigma_2)^2 + 4 \sigma_1\rho_2}\big)$.
\end{lemma}
\newcommand{\proofqCoreGeneral}{
\begin{proof}
	Note that $\alpha \le d^q(i, W)/d^q(i, C')$ for any $i \in N'$, if $\alpha$ is smaller, then $W$ is not in the $\alpha$-$q$-core.
	In particular, we have
	\begin{align*}
		\alpha &\le \min \left( \frac{d^q(i_1, W)}{d^q(i_1, C')}, \frac{d^q(i_1, W)}{d^q(i_2, C')} \right) %
		       \le \min \left( \rho_1 + \rho_2 \frac{d^q(i_2, C')}{d^q(i_1, C')}, \sigma_2 + \sigma_1 \frac{d^q(i_1, C')}{d^q(i_2, C')} \right)\\
		       &\le \max_{x \ge 0} \min \left( \rho_1 + \rho_2 x, \sigma_2 + \sigma_1 / x \right).
	\end{align*}
	As the first term is monotonically increasing in $x$ and the second term is monotonically decreasing in $x$ it suffices to find the maximum $x$ such that the two terms are equal.
	This is true for $x = \frac{1}{2\rho_2}(\sigma_2-\rho_1 + \sqrt{(\rho_1-\sigma_2)^2 + 4\sigma_1\rho_2})$ and yields the desired bound for $\alpha$.
\end{proof}
} \proofqCoreGeneral

\newcommand{\proofrPJR}{
\begin{proof}[Proof of \cref{thm:rPJR}]
	Let $N' \subseteq N$ be a group of agents with $\lvert N'\rvert \ge \ell\frac{n}{k}$ and let $C' \subseteq C$ with $q \le \lvert C' \rvert \le \ell$ such that $d^q(i, C') \le d^q(i, W)$ for any $i \in N'$.
	By \Cref{lem:pigeon} there is a candidate $c$ being ranked in their top $q$ among $C'$ by $q \frac{n}{k}$ many agents $N'' \subseteq N'$. 
	Out of $N''$, let $i_1$ be the agent maximizing $d(i_1, c)$ and $i_2$ be any other agent in $N''$.
	Also, let $C''\subseteq C'$ be the set of the $q$ candidates closest to $i_2$.
	Then, for every $c' \in C''$ and every $i \in N''$, we have \[d(i, c') \le d(i, c) + d(c, i_2) + d(i_2, c') \le d(i_1, c) + 2 d^q(i_2, C') \eqqcolon y.\]
	In other words, $|\bigcap_{i \in N''} B(i, y) \cap C| \ge q$.
	Now \rpjr implies that $|\bigcup_{i \in N''} B(i, y) \cap W| \ge q$.
	Thus, for every $i \in N''$ there is an agent $i' \in N''$ such that $d^q(i, W) \le d(i, i') + y$.
	Since the distance of $i_1$ to any other agent $i' \in N''$ is $d(i_1, i') \le d(i_1, c) + d(c, i_1)$, we have
	\[ d^q(i_1, W) \le 2d(i_1, c) + y \le  3 d^q(i_1, C') + 2 d^q(i_2, C'). \]
	As the distance of $i_2$ to any other agent $i' \in N''$ is at most $d(i', c) + d(c, i_2) \le d(i_1, c) + d^q(i_2, C')$,
	\[ d^q(i_2, W) \le d(i_1, c) + d^q(i_2, C') + y \le 2 d(i_1, c) + 3 d^q(i_2, C') \le 2 d^q(i_1, C') + 3 d^q(i_2, C'). \]
	Applying \cref{lem:q-core-general} with $\rho_1=\sigma_2=3$ and $\rho_2=\sigma_1=2$ yields the stated $5$-$q$-core.
\end{proof}
}  \proofrPJR

\subsection{Dealing with unbounded candidate sets}
\label{sec:polytime}
Whenever the candidate space $C$ is finite, it is straightforward to implement \textsc{greedy capture} and \textsc{expanding approvals} in polynomial time.  
However, as shown by \citet{MiSh20a}, once $C$ is unbounded and the metric space is only implicitly given (e.g., some distance norm over $C = \mathbb R^t$), computing \textsc{greedy capture} can become NP-hard.
For Euclidean distances over $C = \mathbb R^t$, \citet[Theorem~12]{MiSh20a} were nevertheless able to give an approximate version of \textsc{greedy capture}, which approximates proportional fairness up to a factor of $2 + \varepsilon$ for any $\varepsilon > 0$ in this special metric space.
For general metric spaces, \citet[Theorem~11]{MiSh20a} show that in an instance with $N \subseteq C$, any outcome which is $\alpha$-proportionally fair when restricted to the instance with candidate set $N$ is $2 \alpha$-proportionally fair in the whole instance.
\citet{ALM23a} used a very similar approach to this and showed that running \textsc{expanding approvals} on the agents results in a $3$-proportionally fair outcome.
Combining both approaches, we show that any outcome $W$ satisfying \rjr when restricted to the instance with candidate set $N \cup W$ satisfies $3$-proportional fairness in the entire instance.
Thus, \textsc{greedy capture} restricted to the agents yields a $3$-proportionally fair outcome.
The same also applies to \rpjr and the $q$-core. 
\begin{theorem}[
	label=thm:polytime,
	restate=polytime
	]
    Consider an instance $I$ with $N \subseteq C$ and an outcome $W$ and let $I'$ be the instance with agent set $N$ and candidate set $N \cup W$.
    If $W$ satisfies \rjr in $I'$, then $W$ satisfies $3$-proportional fairness.
    If $W$ satisfies \rpjr in $I'$, then $W$ is in the $4$-$q$-core for all $q \le k$.
\end{theorem}
\newcommand{\proofpolytime}{
\begin{proof}
    First consider a group $N'$ deviating to a single candidate $c \in C$ and assume that $W$ satisfies \rjr in $I'$.
    Let $i = \argmin_{i \in N'} d(i,c)$ and $i^* = \argmax_{i^* \in N'} d(i^*, c)$. Then, for any $i' \in N'$,
    \[ d(i,i') \le d(i, c) + d(c, i') \le d(i,c) + d(i^*, c) \eqqcolon y. \] 
    Now, as $W$ satisfies \rjr in $I'$, there must exist a $c' \in W$ and $i'' \in N'$ with $d(i'',c') \le  d(i,c) + d(i^*, c) \le d(i'',c) + d(i^*, c)$. Further, since $d(i*, W) \le d(i^*, c) + d(i'', c) + d(i'', W) \le 2d(i,c) + 2d(i^*, c)$ we can apply \Cref{lem:q-core-general} with $\rho_1 = 2, \rho_2 = 2, \sigma_1 = 1, \sigma_2 = 1$ and obtain that $W$ satisfies $3$ proportional fairness.

    For the $q$-core and \rpjr, we can again apply \Cref{lem:pigeon} and obtain that there must again be a candidate $c$ and a set $N''$ of at least $q\frac{n}{k}$ ranking $c$ among their top $q$ choices. Further, again, let $i_1$ be the agent maximizing $d(i_1, c)$ and $i_2$ be the agent minimizing $d^q(i_2, C')$.
    Now, for every $i, i' \in N''$ we have $d(i', i) \le d(i', c) + d(c, i) \le 2d(i_1,c) \eqqcolon y$. 
    In other words,
    \[ \Big| \bigcap_{i \in N''} B(i, y) \cap C \Big| \ge \Big| \bigcap_{i \in N''} B(i, y) \cap N' \Big| \ge q, \]
    and as $W$ satisfies \rpjr in $I'$, we have $|\bigcup_{i \in N''} B(i, y) \cap W| \ge q$.
    Therefore, we get that $d^q(i_1, W) \le 2d(i_1, c) + y =  4 d(i_1, c) \le 4 d^q(i_1,c)$ which immediatley gives us a $4$ approximation to the $q$-core.
\end{proof}
} \proofpolytime

\section{Stronger fairness bounds for sortition}
\label{sec:sortition}

\citet{EbMi23a} introduced \textsc{fair greedy capture}, a randomized generalization of \textsc{greedy capture} for the setting of sortition.
It works in the setting in which $N = C$ and is parameterized by some parameter $q \le k$.
Like \textsc{greedy capture} it smoothly increases a radius around each agent/candidate.
Once this radius contains at least $q \frac{n}{k}$ agents, it selects $q$ of them uniformly at random and deletes in total $\smash{\lceil q \frac{n}{k} \rceil}$ of these agents.
Together with an adequate final sampling step, one can show that this selects each agent with a probability of exactly $\smash{\frac{k}{n}}$.

\newcommand{\exFGC}{
\begin{example}
	\label{ex:fgc}
    We consider \textsc{fair greedy capture} with parameter $q=2$,
    run on the instance in \cref{fig:example-pr} with $k = 4$. 
    The smallest ball to contain at least $q\frac{n}{k} = 5$ agents is centered either at $6$ or at $9$ and has radius $3$; thus it contains the agents $5, \dots, 10$.
    It selects two of them uniformly at random, and then deletes all but one of them.
    Say the algorithm selects $5$ and $9$ and deletes all but agent $10$.
    The next ball to contain at least $5$ agents is centered at $5$, has radius $10$, and contains agents $1, \dots, 4, 10$.
    Again, the algorithm selects two of the agents uniformly at random.
    So a possible outcome is $\{1, 5, 9, 10\}$.
\end{example}
}

\citeauthor{EbMi23a} show that any clustering returned by the algorithm is in the $\smash{\frac{3 + \sqrt{17}}{2}}$-$1$-core when parameterized by $q=1$
and in the $\smash{\frac{5 + \sqrt{41}}{2}} \approx 5.7$-$q$-core when parameterized by $q > 1$.\footnote{Recall that the algorithm receives $q$ as a parameter in the input; thus, as opposed to our \cref{thm:rPJR}, the bounds only hold for one $q$ at a time.} 
We improve upon their analysis (with a simpler proof) and show that \textsc{fair greedy capture} satisfies a better bound for every parameter $q \le k$.

\begin{theorem}[
	label=thm:FGC,
	restate=FGC
	]
	Let $N=C$ and $q \le k$.
	Then any outcome $W$ returned \textnormal{\textsc{fair greedy capture}} parameterized by $q$ is in the $\frac{3 + \sqrt{17}}{2}$-$q$-core.
\end{theorem}
\newcommand{\proofFGC}{
	\begin{proof}
		Just as in previous proofs, let $N' \subseteq N$ with $|N'| \ge \ell\frac{n}{k}$ and $C' \subseteq C=N$ with $q \le |C'| \le \ell$.
		By \Cref{lem:pigeon} there is a candidate $c$ (which is also an agent) that is marked by a set $N''$ of at least $q \frac{n}{k}$ agents.
		Again, let $i_1 \in N''$ be the agent maximizing $d(i_1,c) \eqqcolon y$.
		Then $d(i, c) \le y$ for every $i \in N''$ and therefore, $N'' \subseteq B(c,y)$.
		This means that \textsc{fair greedy capture} captures parts of $N''$ by a ball of radius at most $y$.
		Let $i_2 \in N''$ be such an agent; suppose that $i_2$ was captured by a ball centered around some point $p$ with radius at most $y$, leading to at least $q$ agents being selected from $B(p,y)$.
		Then
		\[ d^q(i_2, W) \le d(i_2, p) + y \le 2y \le 2d^q(i_1, C'). \]
		Similarly, we can bound
		\[ d^q(i_1, W) \le d(i_1, c) + d(c, i_2) + d^q(i_2, W) \le 3 d^q(i_1, C') + d^q(i_2, C'). \]
		This way, we can invoke \Cref{lem:q-core-general} with $\rho_1 = 3, \rho_2 = 1, \sigma_1 = 2, \sigma_2 = 0$ and obtain the bound of $3 + \frac{ - 3 + \sqrt{9 + 8}}{2} = \frac{3 + \sqrt{17}}{2}$. 
	\end{proof}
} \proofFGC

\section{Conclusion and future work}
In this paper, we studied proportional clustering from a social choice perspective and showed that our new \emph{metric JR axioms} enable near-optimal approximations of fairness notions for clustering.
An interesting open question, both relevant to social choice and clustering is related to a different relaxation of proportional fairness (or core fairness) introduced by \citet{JMW20a}.
Instead of bounding the factor by which the agents can improve, they bound the size of the deviating coalition (similar to the transferable core).
In that sense, no group of size $\gamma \frac{n}{k}$ should exist, who could all deviate to a candidate they like more.
In their work, they show that there are instances for which no solution with $\gamma < 2$ can exist while for any $\varepsilon > 0$ a solution with $\gamma = 16 + \varepsilon$ exists.
Since these results only care about the relative ordering of the candidates, they also translate to clustering.
Closing this bound, or improving it for certain metric spaces, seems like an interesting problem.
It would be also intriguing to study the probabilistic analog of the core \citep{CJMW20a,JMW20a}, especially if the results generalize to the $q$-core and if certain metric spaces admit simple algorithms to compute it.

Further, \textsc{expanding approvals} (\cref{ssec:rpjr}) is more of a family of algorithms, parameterized by how candidates are selected and how budgets are deducted.
Is there any way to axiomatically (or quantitatively) distinguish its different parameterizations?
In the context of approval-based multiwinner voting, the \emph{Method of Equal Shares} \citep{PeSk20a} can be seen as an instantiation of \textsc{expanding approvals} which provides stronger proportionality guarantees than other algorithms in the family.
Is something similar possible for our setting, e.g., can one go from proportionality axioms inspired by PJR to axioms inspired by the stronger EJR axiom \citep{ABC+16a}? 

Naturally, our work still leaves several open questions when it comes to the approximation factors of our notions.
What are the best attainable factors for proportional fairness and the $q$-core?
Further, the questions of \citet{JKL20a} whether the bound of $2$ on individual fairness can be improved for Euclidean spaces and of \citet{MiSh20a} whether for (unweighted) graph metrics with $N = C$ a $1$-proportional fair clustering always exist, are still open.

\bibliographystyle{abbrvnat}
\bibliography{abb, algo2, dist}

\begin{thebibliography}{34}
\providecommand{\natexlab}[1]{#1}
\providecommand{\url}[1]{\texttt{#1}}
\expandafter\ifx\csname urlstyle\endcsname\relax
  \providecommand{\doi}[1]{doi: #1}\else
  \providecommand{\doi}{doi: \begingroup \urlstyle{rm}\Url}\fi

\bibitem[Awasthi et~al.(2022)Awasthi, Brubach, Chakrabarty, Dickerson, Esmaeili, Kleindessner, Knittel, Morgenstern, Samadi, Srinivasan, and Tsepenekas]{ABFair}
P.~Awasthi, B.~Brubach, D.~Chakrabarty, J.~P. Dickerson, S.~A. Esmaeili, M.~Kleindessner, M.~Knittel, J.~Morgenstern, S.~Samadi, A.~Srinivasan, and L.~Tsepenekas.
\newblock Fairness in clustering.
\newblock \url{https://www.fairclustering.com/}, 2022.

\bibitem[Aziz and Lee(2020)]{AzLe20a}
H.~Aziz and B.~E. Lee.
\newblock The expanding approvals rule: improving proportional representation and monotonicity.
\newblock \emph{Social Choice and Welfare}, 54:\penalty0 1--45, 2020.

\bibitem[Aziz and Lee(2021)]{AzLe21a}
H.~Aziz and B.~E. Lee.
\newblock Proportionally representative participatory budgeting with ordinal preferences.
\newblock In \emph{Proceedings of the 35th AAAI Conference on Artificial Intelligence (AAAI)}, pages 5110--5118. AAAI Press, 2021.

\bibitem[Aziz et~al.(2017)Aziz, Brill, Conitzer, Elkind, Freeman, and Walsh]{ABC+16a}
H.~Aziz, M.~Brill, V.~Conitzer, E.~Elkind, R.~Freeman, and T.~Walsh.
\newblock Justified representation in approval-based committee voting.
\newblock \emph{Social Choice and Welfare}, 48\penalty0 (2):\penalty0 461--485, 2017.

\bibitem[Aziz et~al.(2023)Aziz, Lee, {Morota Chu}, and Vollen]{ALM23a}
H.~Aziz, B.~E. Lee, S.~{Morota Chu}, and J.~Vollen.
\newblock Proportionally representative clustering.
\newblock Technical report, arXiv:2304.13917 [cs.LG], 2023.

\bibitem[Bateni et~al.(2024)Bateni, Cohen-Addad, Epasto, and Lattanzi]{BCEL24a}
M.~Bateni, V.~Cohen-Addad, A.~Epasto, and S.~Lattanzi.
\newblock A scalable algorithm for individually fair k-means clustering.
\newblock In \emph{Proceedings of the 27th International Conference on Artificial Intelligence and Statistics (AISTATS)}, 2024.
\newblock Forthcoming.

\bibitem[Brill and Peters(2023)]{BrPe23a}
M.~Brill and J.~Peters.
\newblock Robust and verifiable proportionality axioms for multiwinner voting.
\newblock In \emph{Proceedings of the 24th ACM Conference on Economics and Computation (ACM-EC)}, page 301. ACM, 2023.
\newblock Full version arXiv:2302.01989 [cs.GT].

\bibitem[Brill et~al.(2022)Brill, Israel, Micha, and Peters]{BIMP22a}
M.~Brill, J.~Israel, E.~Micha, and J.~Peters.
\newblock Individual representation in approval-based committee voting.
\newblock In \emph{Proceedings of the 36th AAAI Conference on Artificial Intelligence (AAAI)}, pages 4892--4899. AAAI Press, 2022.

\bibitem[Caragiannis et~al.(2022)Caragiannis, Shah, and Voudouris]{CSV22a}
I.~Caragiannis, N.~Shah, and A.~A. Voudouris.
\newblock The metric distortion of multiwinner voting.
\newblock \emph{Artificial Intelligence}, 313:\penalty0 103802, 2022.

\bibitem[Caragiannis et~al.(2024)Caragiannis, Micha, and Peters]{CMP24a}
I.~Caragiannis, E.~Micha, and J.~Peters.
\newblock Can a few decide for many? {T}he metric distortion of sortition.
\newblock In \emph{Proceedings of the 41st International Conference on Machine Learning (ICML)}, 2024.
\newblock Forthcoming.

\bibitem[Chakrabarti et~al.(2022)Chakrabarti, Dickerson, Esmaeili, Srinivasan, and Tsepenekas]{CDE+22b}
D.~Chakrabarti, J.~P. Dickerson, S.~A. Esmaeili, A.~Srinivasan, and L.~Tsepenekas.
\newblock A new notion of individually fair clustering: $\alpha$-equitable k-center.
\newblock In \emph{Proceedings of the 25th International Conference on Artificial Intelligence and Statistics (AISTATS)}, pages 6387---6408, 2022.

\bibitem[Chen et~al.(2019)Chen, Fain, Lyu, and Munagala]{CFL+19a}
X.~Chen, B.~Fain, L.~Lyu, and K.~Munagala.
\newblock Proportionally fair clustering.
\newblock In \emph{Proceedings of the 36th International Conference on Machine Learning (ICML)}, pages 1032--1041, 2019.

\bibitem[Cheng et~al.(2020)Cheng, Jiang, Munagala, and Wang]{CJMW20a}
Y.~Cheng, Z.~Jiang, K.~Munagala, and K.~Wang.
\newblock Group {{Fairness}} in {{Committee Selection}}.
\newblock \emph{ACM Transactions on Economics and Computation}, 8\penalty0 (4):\penalty0 {23:1--23:18}, 2020.

\bibitem[Chhaya et~al.(2022)Chhaya, Dasgupta, Choudhari, and Shit]{CDCS22a}
R.~Chhaya, A.~Dasgupta, J.~Choudhari, and S.~Shit.
\newblock On coresets for fair regression and individually fair clustering.
\newblock In \emph{Proceedings of the 25th International Conference on Artificial Intelligence and Statistics (AISTATS)}, pages 9603--9625, 2022.

\bibitem[Chierichetti et~al.(2017)Chierichetti, Kumar, Lattanzi, and Vassilvitskii]{CKLV17a}
F.~Chierichetti, R.~Kumar, S.~Lattanzi, and S.~Vassilvitskii.
\newblock Fair clustering through fairlets fair clustering through fairlets.
\newblock In \emph{Proceedings of the 30th Conference on Neural Information Processing Systems (NeurIPS)}, pages 5029--5037, 2017.

\bibitem[Dickerson et~al.(2023{\natexlab{a}})Dickerson, Esmaeili, Morgenstern, and Zhang]{DEMZ23a}
J.~P. Dickerson, S.~A. Esmaeili, J.~Morgenstern, and C.~J. Zhang.
\newblock Doubly constrained fair clustering.
\newblock In \emph{Proceedings of the 37th Conference on Neural Information Processing Systems (NeurIPS)}, pages 13267--13293, 2023{\natexlab{a}}.

\bibitem[Dickerson et~al.(2023{\natexlab{b}})Dickerson, Esmaeili, Morgenstern, and Zhang]{DEMZ24a}
J.~P. Dickerson, S.~A. Esmaeili, J.~Morgenstern, and C.~J. Zhang.
\newblock Fair clustering: Critique, caveats, and future directions.
\newblock NeurIPS 2023 Workshop on Algorithmic Fairness through the Lens of Time, 2023{\natexlab{b}}.

\bibitem[Dummett(1984)]{Dumm84a}
M.~Dummett.
\newblock \emph{Voting Procedures}.
\newblock Oxford University Press, 1984.

\bibitem[Ebadian and Micha(2023)]{EbMi23a}
S.~Ebadian and E.~Micha.
\newblock Boosting sortition via proportional representation.
\newblock Unpublished manuscript. Available at \url{http://www.cs.toronto.edu/~emicha/papers/boosting-sortition.pdf}, 2023.

\bibitem[Ebadian et~al.(2022)Ebadian, Kehne, Micha, Procaccia, and Shah]{EKM+22a}
S.~Ebadian, G.~Kehne, E.~Micha, A.~D. Procaccia, and N.~Shah.
\newblock Is sortition both representative and fair?
\newblock In \emph{Proceedings of the 36th Conference on Neural Information Processing Systems (NeurIPS)}, pages 3431--3443, 2022.

\bibitem[Flanigan et~al.(2021)Flanigan, G{\"o}lz, Gupta, Hennig, and Procaccia]{FGG+21a}
B.~Flanigan, P.~G{\"o}lz, A.~Gupta, B.~Hennig, and A.~D. Procaccia.
\newblock Fair algorithms for selecting citizens' assemblies.
\newblock \emph{Nature}, 596:\penalty0 548--552, 2021.

\bibitem[Han et~al.(2023)Han, Xu, Xu, and Yang]{HXXY23a}
L.~Han, D.~Xu, Y.~Xu, and P.~Yang.
\newblock Approximation algorithms for the individually fair k-center with outliers.
\newblock \emph{Journal of Global Optimization}, 87\penalty0 (2):\penalty0 603--618, 2023.

\bibitem[Jiang et~al.(2020)Jiang, Munagala, and Wang]{JMW20a}
Z.~Jiang, K.~Munagala, and K.~Wang.
\newblock Approximately stable committee selection.
\newblock In \emph{Proceedings of the 52nd Annual ACM SIGACT Symposium on Theory of Computing (STOC)}, pages 463--472. ACM, 2020.

\bibitem[Jung et~al.(2020)Jung, Kannan, and Lutz]{JKL20a}
C.~Jung, S.~Kannan, and N.~Lutz.
\newblock Service in your neighborhood: Fairness in center location.
\newblock In \emph{Proceedings of the 1st Symposium on Foundations of Responsible Computing (FORC)}, pages 5:1---5:15, 2020.

\bibitem[Kalayc{\i} et~al.(2024)Kalayc{\i}, Kempe, and Kher]{KKK24b}
Y.~H. Kalayc{\i}, D.~Kempe, and V.~Kher.
\newblock Proportional representation in metric spaces and low-distortion committee selection.
\newblock In \emph{Proceedings of the 38th AAAI Conference on Artificial Intelligence (AAAI)}, pages 9815--9823, 2024.

\bibitem[Kar et~al.(2023)Kar, Kosan, Mandal, Medya, Silva, Dey, and Sanyal]{KKM+23a}
D.~Kar, M.~Kosan, D.~Mandal, S.~Medya, A.~Silva, P.~Dey, and S.~Sanyal.
\newblock Feature-based individual fairness in k-clustering.
\newblock In \emph{Proceedings of the 22nd International Conference on Autonomous Agents and Multiagent Systems (AAMAS)}, pages 2772--2774, 2023.

\bibitem[Lackner and Skowron(2022)]{LaSk22a}
M.~Lackner and P.~Skowron.
\newblock \emph{Multi-Winner Voting with Approval Preferences}.
\newblock Springer, 2022.

\bibitem[Li et~al.(2021)Li, Li, Sun, Wang, and Wang]{LLS+21a}
B.~Li, L.~Li, A.~Sun, C.~Wang, and Y.~Wang.
\newblock Approximate group fairness for clustering.
\newblock In \emph{Proceedings of the 38th International Conference on Machine Learning (ICML)}, pages 6381--6391, 2021.

\bibitem[Mahabadi and Vakilian(2020)]{MaVa20a}
S.~Mahabadi and A.~Vakilian.
\newblock Individual fairness for k-clustering.
\newblock In \emph{Proceedings of the 37th International Conference on Machine Learning (ICML)}, pages 6586--6596, 2020.

\bibitem[Micha and Shah(2020)]{MiSh20a}
E.~Micha and N.~Shah.
\newblock Proportionally fair clustering revisited.
\newblock In \emph{Proceedings of the 47th International Colloquium on Automata, Languages, and Programming (ICALP)}, pages 85:1--85:16, 2020.

\bibitem[Negahbani and Chakrabarty(2021)]{NeCh21a}
M.~Negahbani and D.~Chakrabarty.
\newblock Better algorithms for individually fair $ k $-clustering.
\newblock In \emph{Proceedings of the 34th Conference on Neural Information Processing Systems (NeurIPS)}, pages 13340--13351, 2021.

\bibitem[Peters and Skowron(2020)]{PeSk20a}
D.~Peters and P.~Skowron.
\newblock Proportionality and the limits of welfarism.
\newblock In \emph{Proceedings of the 21st ACM Conference on Economics and Computation (ACM-EC)}, pages 793--794. ACM, 2020.

\bibitem[Sternbach and Cohen(2023)]{StCo23a}
H.~Sternbach and S.~Cohen.
\newblock Fair facility location for socially equitable representation.
\newblock In \emph{Proceedings of the 22nd International Conference on Autonomous Agents and Multiagent Systems (AAMAS)}, pages 2775--2777, 2023.

\bibitem[Vakilian and Yal{\c c}ıner(2022)]{VaYa22b}
A.~Vakilian and M.~Yal{\c c}ıner.
\newblock Improved approximation algorithms for individually fair clustering.
\newblock In \emph{Proceedings of the 25th International Conference on Artificial Intelligence and Statistics (AISTATS)}, pages 8758--8779, 2022.

\end{thebibliography}

\appendix

\section{Generalizing the transferable core and individual fairness}
\label{app:q:gem}

In the spirit of the $\alpha$-$q$-core, we introduce generalizations of individual fairness and the transferable core.
Generalizing the former is straightforward: The $q$-th closest center must be within the radius of the closest $q \frac{n}{k}$ agents.
For this, let $r^q_{N, k} = r_{N, k/q} = \min\{r \in \mathbb R : |B(i, r) \cap N| \ge q\frac{|N|}{k}\}$.
Just as we did for the $\beta$-individual fairness, we require $N \subseteq C$ for this definition to work.
This time, we also require that $k \le n$; otherwise the agent may deserve an outcome containing multiple candidates in the same location.%
\begin{definition}
    \label{def:qIF}
    For $\beta \ge 1$ and for an instance with $N \subseteq C$ and $k \le n$, an outcome $W$ satisfies \emph{$\beta$-$q$-individual fairness} if $d^q(i, W) \le \beta r^q_{N, k}$ for all $i \in N$.
\end{definition}

For the transferable core, we propose the following generalization,
in which a group of $\gamma q \lceil\frac{n}{k}\rceil$ agents diverges to another outcome if their average distance to the $q$-th furthest center would improve by a factor of $\alpha$.
\begin{definition}
    \label{def:qTC}
    An outcome $W$ is in the \emph{$(\gamma, \alpha)$-$q$-transferable core} if there is no $\ell \in \mathbb N$, no group of agents $N' \subseteq N$ with $\lvert N'\rvert \ge \gamma \ell \frac{n}{k}$, and no set of candidates $C' \subseteq C$ with $q \le \lvert C'\rvert \le \ell$ such that 
    $ %
    \alpha \sum_{i \in N'} d^q(i, C') < \sum_{i \in N'} d^q(i, W).
    $ %
\end{definition}

As mentioned in related work, \citet{KKK24b} proposed a different generalization of the transferable core.
They fix $q = \ell$, and, consider the sum over the first $q$ closest points, i.e., they replace $d^q(i, \cdot)$ with $\sum_{j=1}^q d^j(i, \cdot)$.

Again, we first establish relations between the notions.
We show that the approximation guarantees for $q = 1$ nearly generalize to arbitrary $q$.
An important distinction we need to make here, is that for the $q$-transferable core, we need $\ell$ --- the number of candidates the agents deviate to --- to be less than $2q$.
For $\ell \ge 2q$, no such approximation would be possible; we show this in an example afterwards.

\begin{theorem}[
	label=thm:qC-IF-TC,
	restate=QFIFTC
	]
    For any $q \le k$, any outcome in the $\alpha$-$q$-core
    is $(2\alpha + 1)$-$q$-individually fair if $N \subseteq C$ and $k \le n$ and
    in the $(\gamma, \frac{\gamma(\alpha+1)}{\gamma - 1})$-$q$-transferable core if $2q>\ell$.
    Further, if $N \subseteq C$, any $\beta$-individually fair outcome is in the $2\beta$-$q$-core. 
\end{theorem}
\newcommand{\proofQFIFTC}{
\begin{proof}
    Let $W$ be in the $\alpha$-$q$-core and $i \in N'$. Further, let $x = r^q_{N, k}(i)$. Thus, if the $q \frac{n}{k}$ closest agents to $i$ would deviate to $q$ of them, their $q$-distance would be at most $2x$. Hence, by the $\alpha$-$q$-core, there must be one agent of them, with a $q$-distance of $\alpha 2x$ to $W$ and thus the $q$-distance of $i$ to $W$ is at most $(2\alpha + 1)x$.

    For the $q$-transferable core for $2q>\ell$, consider a set $N' \subseteq N$ with $n' = |N'| \ge \gamma \ell \frac{n}{k}$ and let $C' \subseteq C$ with $q \le |C'| \le \ell$.
    Our proof is a simple generalization of the proof of \cref{thm:PF-TC}:
    We let $\eta = \lceil \ell\frac{n}{k} \rceil$ and choose agents $j_\lambda \in J_\lambda$, $\lambda = 0, \dots, n'-\eta$ such that $d^q(j_\lambda, W) \le \alpha d^q(j_\lambda, C')$;
    such an agent is guaranteed to exist as $W$ is in the $\alpha$-$q$-core.
    Now, for each $i \in N'' = N' \setminus \{j_0, \dots, j_{n'-\eta}\}$, we need to bound $d^q(i, W)$.
    As $|C'| \le \ell < 2q$, among the $q$ candidates in $C'$ closest to $i$ and those closest to $j_0$, we find at least one common candidate $c$ due to the pigeonhole principle.
    That is, there exists a candidate $c \in C'$ such that $d(i, c) \le d^q(i, C')$ and $d(j_0, c) \le d^q(j_0, C')$.
    Therefore, in analogy to \eqref{eq:TC1}, we have
    \[ d^q(i, W) \le d(i, c) + d(j_0, c) + d^q(j_0, W) \le d^q(i, C') + (\alpha+1) d^q(i_\eta, C'). \]
    As again $\frac{|N''|}{n'-|N''|} \le \frac{1}{\gamma-1}$, we obtain a $(\gamma, \frac{\gamma(\alpha+1)}{\gamma-1})$-$q$-transferable core.

    Next, let $W$ be a $\beta$-individually fair outcome. Let $N'$ and $C'$ witness a $q$-core violation. As in the previous proofs, there must be a candidate $c \in C'$ such that at least $\frac{qn}{k}$ agents  $N'' \subseteq N'$ who have this candidate in their first $q$. Let $i = \argmax_{i \in N'} d(i,c)$. Then since $W$ satisfies $\beta$-individual representation, it must hold that
    \[d^q(i, W) \le \beta \max_{i' \in N''} d(i, i') \le \beta (d(i,c) + \max_{i' \in N'} d(i',c)) \le 2 \beta d^q(i,C').\qedhere\]
\end{proof}
} \proofQFIFTC

Indeed, the bound on $\ell$ for the $q$-transferable core is necessary, without it instances might exist, which satisfy all of our proportionality notions, but no approximation to the $q$ core for any finite approximation factor.
Recall that UPRF (see \cref{def:UPRF} in \cref{app:uprf}) is one of the two notions introduced by \citet{ALM23a}.

\begin{theorem}[label=thm:qTC-lb,restate=qTCLB]
    There exists an instance with $N=C$ with an outcome that satisfies \rpjr, UPRF, and $1$-$q$-individual fairness, is in the $1$-$q$-core, and in the support of \textnormal{\textsc{fair greedy capture}} parameterized by $q$; but, it is not in the $(\gamma, \alpha)$-$q$-transferable core for $\ell \ge 2q$, any $\gamma \ge 1$, and any value of $\alpha$.
\end{theorem}
\newcommand{\proofqTCLB}{
\begin{proof}
    Our instance contains two sets $N_1, N_2$ of $\lceil \frac{n}{k} \rceil$ and $\lfloor (k-1) \frac{n}{k} \rfloor$ agents, respectively.
    All agents in $N_1$ lie on one point, and all agents in $N_2$ lie on a point with distance $1$ to the agents in $N_1$, i.e., $d(i, j) = 1$ if $i \in N_1$ and $j \in N_2$ and $d(i, j) = 0$ otherwise.
    Our outcome $W$ contains one agent in $N_1$ and $(k-1)$ agents in $N_2$.
    The outcome satisfies \rpjr:
    For $y=0$, as $B(i, 0) \cap N_x = \emptyset$ if $i \not\in N_x$, we only have to consider sets $N' \subseteq N_1$ and $N' \subseteq N_2$; and for these sets, there are sufficiently many centers close by.
    As for $y=1$ we have $N \subseteq B(i, 1)$ for every $i \in N$; therefore we satisfy \rpjr.
    The argumentation for UPRF is identical. Further, by a similar argument, as both sets cannot deviate, the outcomes satisfies both the $1$-$q$-core and 1-$q$-individual fairness. It can also be selected by \textsc{fair greedy capture}, as \textsc{fair greedy capture} can select arbitrary candidates after deleting all agents.

    Now consider a set $N' \subseteq N$ with $|N'| \ge 2q \cdot \gamma \cdot \frac{n}{k}$ that contains an agent $j \in N' \cap N_1$ and a set $C'$ that contains $q$ agents from $N_1$ and at least $q$ agents from $N_2$. 
    As there is only one center within distance $0$ from $j$, we have $d^q(j, W) = 1$; however,
    as $|C' \cap N_2| \ge |C' \cap N_1| = q$, we have $d^q(i, C') = 0$ for every $i \in N'$.
    Therefore, for any choice of $\alpha$,
    \[ \sum_{i \in N'} d^q(i, W) \ge 1 > \alpha \sum_{i \in N'} d^q(i, C') = 0. \qedhere \]
\end{proof}
} \proofqTCLB

Next, we prove the bounde on the $q$-individual fairness and the $q$-transferable core implied by \rpjr.

\begin{theorem}[label=thm:rPJR2,restate=rPJR2]
    If an outcome satisfies \rpjr,
    then, for every $q \le k$, it
    \begin{inparaenum}[(1)]
        \item~satisfies $3$-$q$-individual fairness if $N \subseteq C$ and $k \le n$;
        \item~is in the $(\gamma, \frac{3\gamma + 1}{\gamma-1})$-$q$-transferable core for any $\gamma > 1$ if $2q>\ell$.
    \end{inparaenum}
\end{theorem}
\newcommand{\proofrPJRSecond}{
\begin{proof}
    For $q$-individual fairness, let $x = r^q_{N, k}(i)$.
    Then the maximum distance between any two of the closest $q \frac{n}{k}$ agents to $i$ (including $i$) is $2x$.
    Thus, \rpjr requires that there are at least $q$ candidates with a distance of at most $2x$ from any of these agents and therefore $d^q(i, W) \le 3x$.

    Finally, we prove the bound for the $q$-transferable core for the case that $2q>\ell$.
    Let $N' \subseteq N$ be a group of at least $n' \ge \gamma q \frac{n}{k}$ agents, and let $C' \subseteq C$ with $q \le |C'| \le \ell$.
    Without loss of generality we can assume that $n' = \lceil \gamma \ell \frac{n}{k} \rceil$.
    Let $\eta = \lceil \ell\frac{n}{k} \rceil$,
    and let $N' = \{i_1, \dots, i_{n'}\}$ be ordered such that $d^q(i_z, C') \le d^q(i_{z+1}, C')$ for all $z \in [n'-1]$.
    As in the above proof for the $q$-core, we have the agents $i_1, \dots, i_\eta$ mark their top $q$ candidates.
    As $|C'| \le \ell$ and $q \eta \ge q \ell \frac{n}{k}$, there is at least one candidate $c$ which is marked by a set $N''$ of at least $q\frac{n}{k}$ agents.
    Let $y = d^q(i_\eta, C')$.
    Then $|\bigcap_{j \in N''} B(j, y) \cap C| \ge q$, therefore by \rpjr, $|\bigcup_{j \in N''} B(j, y) \cap W| \ge q$.
    This allows us to provide a bound on $d^q(i, W)$ for every $i \in N'$:
    As $|C'| \le \ell < 2q$, we have that $B(i, d^q(i, C')) \cap C'$ intersects $B(j, y) \cap C'$ for every $j \in N''$.
    As there are at least $q$ centers within a distance of $y$ from the agents in $N''$, we have
    $d^q(i, W) = d^q(i, C') + 2y$.
    Therefore,
    \[
        \sum_{i \in N'} d^q(i, W) \le 2n'y + \sum_{i \in N'} d^q(i, C').
    \]
    As $y = d^q(i_\eta, C') \le d^q(i_z, C')$ for all $\eta \le z \le n'$,
    we have
    \[ 2n'y \le 2n' \cdot \frac{1}{n'-\eta+1} \sum_{z=\eta}^{n'} d^q(i_z, C') \le \frac{2n'}{n'-\eta+1} \sum_{i \in N'} d^q(i, C'). \]
    Finally, as $\gamma q \frac{n}{k} \le n' \le \gamma q \frac{n}{k}+1$ and $\eta-1 \le q \frac{n}{k}$, we obtain a $(\gamma, \alpha)$-$q$-transferable core with
    \[ \alpha \le 1 + \frac{2n'}{n'-\eta+1} \le 1 + \frac{2(\gamma q \frac{n}{k}+1)}{(\gamma-1) q\frac{n}{k}} \le 1 + \frac{2\gamma}{\gamma-1} + \frac{2}{(\gamma-1)q\frac{n}{k}} \le \frac{3\gamma + 2 \frac{k}{qn} - 1}{\gamma-1}. \]
    If $\frac{k}{qn} > 1$, then, by \rpjr, the outcome $W$ must contain the $q$ closest candidates for every agent.
    Therefore $d^q(i, W) = d^q(i, C) \le d^q(i, C')$ for every agent $i$, and we obtain a trivial $(\gamma, 1)$-$q$-core.
    If $\frac{k}{qn} \le 1$, then the above bound yields $\alpha \le \frac{3\gamma+1}{\gamma-1}$.
\end{proof}
}  \proofrPJRSecond

Finally, we prove bounds on the $q$-individual fairness and $q$-transferable core that are fulfilled by any outcome produced by the \textsc{fair greedy capture} algorithm for the setting of sortition (see \cref{sec:sortition}).

\begin{theorem}[label=thm:FGC2,restate=FGC2]
	If $N=C$, then an outcome of \textnormal{\textsc{fair greedy capture}} parameterized by $q \le k$
    \begin{inparaenum}[(1)]
        \item satisfies $3$-$q$-individual fairness if $k \le n$;
        \item is in the $(\gamma, \frac{5\gamma}{\gamma-1})$-$q$-transferable core for any $\gamma > 1$ and if $2q>\ell$.
    \end{inparaenum}
\end{theorem}
\newcommand{\proofFGCtwo}{
\begin{proof}
    For $q$-individual fairness, it follows similarly, that since there is a ball of radius $r^q_{N, k}(i)$ covering the next $\frac{qn}{k}$ agents to $i$, that one of them must be contained in a ball of a not larger radius and thus, $d^q(i,W) \le 3r^q_{N, k}(i)$.

    For the $q$-transferable core for $2q>\ell$, consider a set $N' \subseteq N$ with $n' = |N'| \ge \gamma \ell \frac{n}{k}$ and let $C' \subseteq C=N$ with $q \le |C'| \le \ell$.
    We again follow the proof of \Cref{thm:qC-IF-TC}.
    We let $\eta = \lceil \ell\frac{n}{k} \rceil$ and choose agents $j_\lambda \in J_\lambda$, $\lambda = 0, \dots, n'-\eta$ in a similar manner as done for the $q$-core:
    After having the agents in $J_\lambda$ put marks on their top $q$ candidates, there is a candidate $c$ (which is also an agent) that is marked by a set $J'_\lambda$ of at least $q \frac{n}{k}$ agents.
    As for every $i \in J'_\lambda$, $d(i, c) \le d^q(i, C') \le d^q(i_{\eta+\lambda}, C')$, we know that $J'_\lambda$ is contained in a ball of radius at most $y = 2d^q(i_{\eta+\lambda}, C')$.
    Thus, \textsc{fair greedy capture} captures a ball of radius at most $y$ which contains at least one agent from $J'_\lambda$: this is our agent $j_\lambda$.
    As the captured ball will contain at least $q$ centers, we have
    \[ d^q(j_0, W) \le 2y \le 4d^q(i_{\eta+\lambda}, C'). \]
    It remains to bound $d^q(i, W)$ for each $i \in N'' = N' \setminus \{j_0, \dots, j_{n'-\eta}\}$.
    As $|C'| \le \ell < 2q$, among the $q$ candidates in $C'$ closest to $i$ and those closest to $j_0$, we find at least one common candidate $c$ due to the pigeonhole principle.
    That is, there exists a candidate $c \in C'$ such that $d(i, c) \le d^q(i, C')$ and $d(j_0, c) \le d^q(j_0, C')$.
    Therefore,
    \[ d^q(i, W) \le d(i, c) + d(j_0, c) + d^q(j_0, W) \le d^q(i, C') + 5 d^q(i_\eta, C'). \]
    As again $\frac{|N''|}{n'-|N''|} \le \frac{1}{\gamma-1}$, we get the stated $q$-transferable core by bounding $\sum_{i \in N'} d^q(i, W)$ in the same manner as in \eqref{eq:TC2}, replacing $\alpha$ by $4$ and $d(i, c)$ by $d^q(i, C')$.
\end{proof}
}\proofFGCtwo

\section{Unconstrained Proportionally Representative Fairness (UPRF)}
\label{app:uprf}

\citet{ALM23a} introduced UPRF as a proportional fairness axiom that is always satisfiable and efficiently computable when the candidate space is unbounded.
It is defined as follows.

\begin{definition}
    \label{def:UPRF}
    An outcome $W$ satisfies \emph{Unconstrained Proportionally Representative Fairness (UPRF)}
    if there is no $\ell \in \mathbb{N}$, no group $N' \subseteq N$ with size $\lvert N'\rvert \ge \ell \frac{n}{k}$, and no distance threshold $y \in \mathbb{R}$ such that $\max_{i, i' \in N'} d(i, i') \le y$ and 
    $%
    \left \lvert \bigcup_{i \in N'} B(i, y) \cap W\right\rvert < \ell. 
    $%
\end{definition}

\newcommand{\exDefDPRF}{
\begin{example}
    \label{ex:dprf}
    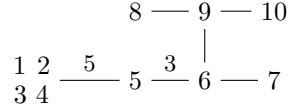
\begin{figure}
        \centering
        \begin{tikzpicture}[scale=0.92]
            \begin{scope}[yscale=-1,xshift=4.5cm]

                \node [align=center] (v1) at (.5,0.5) {$1\,\ 2$\\$3\,\,4$};
                
                \node (v4) at (2,0.5) {$5$};
                
                \node (v5) at (3,0.5) {$6$};
                
                \node (v6) at (4,0.5) {$7$};

                \node (v7) at (2,-0.5) {$8$};

                \node (v8) at (3,-0.5) {$9$};

                \node (v9) at (4,-0.5) {$10$};
                
                \draw (v1) -- node[above] {\footnotesize $5$} (v4);
                \draw (v4) -- node[above] {\footnotesize $3$} (v5);
                \draw (v5) -- (v6);
                \draw (v5) -- (v8);
                \draw (v7) -- (v8);
                \draw (v8) -- (v9);
            \end{scope}
        \end{tikzpicture}
	\caption{Metric space for \cref{ex:dprf}. Edges without labels have length $1$, the distance between any two points is given by the length of the shortest path between them.}
        \label{fig:example-dprf}
    \end{figure}
    Consider the instance depicted in \cref{fig:example-dprf} with $k = 5$. Again, consider the clustering $\{1,2,3,6,9\}$. This outcome does not satisfy \rpjr, since for $y = 4$, it holds that $\bigcap_{i = 5}^{10} B(i, 4) = \{6,7,9\} \subseteq \{5, \dots, 10\} = \bigcup_{i = 5}^{10} B(i, 4)$, however from $\{5, \dots, 10\}$ only $2$ agents instead of the required $3$ got selected into the clustering. 
    It does however satisfy UPRF. One can easily check this for $\ell = 1,2$. For $\ell = 3$, in the group $\{5, \dots, 10\}$ the distance between the agents $5$ and $10$ is $5$. Thus, UPRF only requires $3$ agents to be selected which are at a distance at most $5$ away from one of these agents. Since $1$ to $4$ are at a distance of $5$ to agent $5$ this is satisfied.
\end{example}
} \exDefDPRF

In terms of UPRF, we can characterize clustering instances into two meaningful groups:
Those in which UPRF is implied by \rpjr and those in which an outcome satisfying UPRF need not exist.
We remark that the former group contains all instances with $N \subseteq C$ and $k \le n$, which, as argued above, are two plausible restrictions in many settings.

\begin{observation}[
	label=obs:dprf-uprf,
	restate=DPRFUPRF
	]
	Consider an instance in which for each agent $i \in N$ there are at least $\frac{k}{n}$ unique candidates at distance $0$ to $i$.
	Then any outcome satisfying \rpjr also satisfies UPRF.
	In any instance that does not have this property, an outcome satisfying UPRF need not exist.
\end{observation}
\newcommand{\proofDPRFUPRF}{
\begin{proof}
    For the first part,
    consider an outcome $W$ that does not satisfy UPRF, i.e., there is a group $N' \subseteq N$ with $|N'| \ge \ell\frac{n}{k}$ with $y = \max_{i, i' \in N'} d(i, i')$ and $|\bigcup_{i \in N'} B(i, y) \cap W| < \ell$.
    As we have at least $\frac{k}{n}$ candidates per agent at distance $0$, we have $|\bigcap_{i \in N'} B(i, y) \cap C| \ge |N'| \cdot \frac{k}{n} \ge \ell$.
    Thus, $W$ also does not satisfy \rpjr.

    For the second part, consider an instance with $n < k$, in which there is an agent $i$ and less than $\frac{k}{n}$ candidates with distance $0$ to $i$.
    Then the set $N' = \{i\}$ has size at least $1 \ge \ell\frac{n}{k}$ for $\ell \ge \frac{k}{n} > 1$.
    As $\max_{i, i' \in N'} d(i, i') = 0$, UPRF requires for any outcome $W$ that $|B(i, 0) \cap W| \ge \ell \ge \frac{k}{n}$, but there aren't sufficiently many candidates in $B(i,0)$.
\end{proof}
}\proofDPRFUPRF

We further observe that, even if $N \subseteq C$ and $k \le n$, UPRF is not equivalent to \rpjr.
We can even show a slightly stronger statement (recall that \rpjr implies \rjr).

 \begin{observation}[
	 label=obs:uprf-ne-jr,
	 restate=UPRFneJR
	 ]
    In instances with $N \subseteq C$ and $k \le n$,
    there exist outcomes that satisfy UPRF but not \rjr.
\end{observation}
\newcommand{\proofUPRFneJR}{
\begin{proof}
    Consider an instance with $k = 1$ and $n = 3$ agents $\{1, 2, 3\}$ and one candidate $c$ on a path, such that $d(1, 2) = d(2, 3) = 1$ and $d(3, c) = 2$ with all other distances induced by this path.
    Then the outcome $W = \{c\}$ satisfies UPRF, since the distance between $1$ and $3$ is $2$.
    However, it does not satisfy \rjr,
    as the three agents all have candidate $2$ within distance $1$.
\end{proof}
}\proofUPRFneJR
For the sake of completeness, we prove bounds on the several fairness measures implied by UPRF. 
We first show that UPRF implies bounds on proportional and individual fairness and the transferable core.

\begin{theorem}[
	label=thm:uprf-single-cand,
	restate=UPRFsinglecand
	]
    An outcome satisfying UPRF also satisfies
    $\frac{3 + \sqrt{17}}{2}$-proportional fairness, $3$-individual fairness, and is in the $(\gamma, \frac{3\gamma}{\gamma-1})$-transferable core for any $\gamma > 1$. 
\end{theorem}
\begin{proof}
    The proportional fairness was shown by \citet{ALM23a}. 

    For individual fairness, we again assume that $N \subseteq C$.
    Let $i \in N$ be any agent and let $N' \subseteq N$ be a group of at least $\frac{n}{k}$ agents closest to $i$ (including $i$).
    Let $y = \max_{i' \in N'} d(i, i')$. Then, the maximum distance between any two points in $N'$ is at most $2y$ and UPRF implies that there must be a $j \in N'$ and a candidate $c' \in W$
    such that $c' \in B(i, 2y) \cap W$.
    By the triangle inequality, $d(i, W) \le d(i, j) + d(j, W) \le 3 \max_{i' \in N'} d(i, i')$;
    thus $W$ satisfies $3$-individual fairness.

    For the transferable core, let $N' \subseteq N$ be a group of size $n' \ge \gamma \frac{n}{k}$ and let $c \notin W$.
    Let $\eta = \lceil \frac{n}{k} \rceil$.
    We can now use the same argumentation as in the proof of \Cref{thm:PF-TC}:
    Again, we order the agents $i_1, \dots, i_{n'}$ in $N'$ by their increasing distance to $c$.
    We choose the elements $j_\lambda \in J_\lambda$, $\lambda = 0, \dots, n' - \eta$, such that $d(j_\lambda, W) \le 2d(i_{\eta+\lambda}, c)$;
    such a $j_\lambda$ always exists by UPRF, as $|J_\lambda| = \eta$.
    By replacing every occurrence of $\alpha$ by $2$ in the inequality in the proof of \Cref{thm:PF-TC}, we obtain the claimed bound.
\end{proof}

Concerning the tightness of the bounds, we observe that \citet{ALM23a} already proved the tightness of the proportional fairness bound. For individual fairness, we can show that the bound is tight, while for the transferable core we can derive the same bound as in \Cref{thm:lb:tq}.

\begin{theorem}[
	label=thm:uprf-single-cand-lb,
	restate=UPRFsinglecandLB
	]
    For any $\varepsilon > 0$ there exist instances with outcomes that satisfy UPRF but are not $\frac{3 + \sqrt{17}}{2} - \varepsilon$ proportional fair, $3 - \varepsilon$-individual fair, or in the $(\gamma, \frac{\gamma + 1}{\gamma - 1} - \varepsilon)$-transferable core.
\end{theorem}
\begin{proof}
    Firstly, the case for proportional fairness follows from \citet[Proposition~1]{ALM23a}. For individual fairness, again consider an instance with $k = 1$, three agents $i_1, i_2, i_3$  and one candidate $c_1$ on a path, such that $d(i_1, i_2) = d(i_2, i_3) = 1$ and $d(i_3, c_1) = 2$. As argued before, just selecting $c_1$ satisfies UPRF. However, it only satisfies $3$-individual fairness, since it is at distance $3$ to agent $i_2$ who is at distance $1$ to both $i_1$ and $i_3$.
    For the transferable core, we observe that the instance described in \Cref{thm:lb:tq} with $\alpha = 1$ also applies to UPRF.
\end{proof}

Similarly, under the condition that $N \subseteq C$, we derive bounds for UPRF towards the fairness measures for deviating to multiple candidates.
Indeed, the bound on the $q$-core matches the ont for \rpjr.
Nevertheless, for the $q$-transferable core, we were not able to obtain an equally good bound as for for \rpjr.
\begin{theorem}[
	label=thm:UPRF-multi-cand,
	restate=UPRFmulticand
	]
    In an instance with $N \subseteq C$ let $W$ be an outcome that satisfies UPRF. Then, for any $q \le k$,
    \begin{compactenum}[(1)]
        \item $W$ is in the $4$-$q$-core;
        \item $W$ satisfies $3$-$q$-individual fairness if $k \le n$;
        \item $W$ is in the $(\gamma, \frac{5\gamma + 1}{\gamma-1})$-$q$-transferable core for any $\gamma > 1$ and if $2q>\ell$.
    \end{compactenum}
\end{theorem}
\begin{proof}
    We start with the $\alpha$-$q$-core. This follows almost identically to \Cref{thm:polytime}. Following \Cref{lem:pigeon} there again exists a candidate $c$ ranked in the top-$q$ by at least $q\frac{n}{k}$ many agents $N''$
    Out of these, let $i_1$ be the agent maximizing $d(i_1, c)$ and $i_2$ be the agent minimizing $d^q(i_2, C')$.
    Now, for every $i, i' \in N''$ we have $d(i', i) \le d(i', c) + d(c, i) \le 2d(i_1,c) \eqqcolon y$. 
   Thus, since $W$ satisfies UPRF we have $|\bigcup_{i \in N''} B(i, y) \cap W| \ge q$.
    Therefore, we get that $d^q(i_1, W) \le 2d(i_1, c) + y =  4 d(i_1, c) \le 4 d^q(i_1,c)$ which immediatley gives us a $4$ approximation to the $q$-core.

    For $q$-individual fairness, let $x = r^q_{N, k}(i)$ for an agent $i \in N'$.
    Then the maximum distance between any two of the closest $q \frac{n}{k}$ agents to $i$ (including $i$) is $2x$.
    Thus, UPRF requires that there are at least $q$ candidates within a distance of at most $2x$ from any of these agents and therefore $d^q(i, W) \le 3x$.

    Finally, for the $q$-transferable core with $2q>\ell$, we again follow the proof of \Cref{thm:rPJR}.
    Letting $N' \subseteq N$ with $|N'| = n' = \lceil\gamma \ell\frac{n}{k}\rceil$ and $C' \subseteq C$ with $q \le |C'| \le \ell$,
    we again have the agents $i_1, \dots, i_\eta$ pick their top $q$ candidates and will find a candidate $c$ marked by a set $N''$ of at least $q \frac{n}{k}$ agents.
    Let $y = d^q(i_\eta, C')$.
    Then the largest distance between any two agents in $N''$ is at most $2y$; therefore by UPRF, $|\bigcup_{j \in N''} B(j, 2y) \cap W| \ge q$.
    Therefore, in analogy to the other proof, we have $d^q(i, W) \le d^q(i, C') + 4y$;
    thus we obtain a $(\gamma, \alpha)$-$q$-transferable core with
    \[ \alpha \le 1 + \frac{4n'}{n'-\eta+1} \le \frac{5\gamma + 2 \frac{k}{qn} - 1}{\gamma - 1}. \]
    As we can again assume that $\frac{k}{qn} \le 1$ we obtain a bound of $\frac{5\gamma + 1}{\gamma-1}$.
\end{proof}

\end{document}